\documentclass{article}

\usepackage{microtype}
\usepackage{graphicx}
\usepackage{caption}
\usepackage{subcaption}
\usepackage{booktabs} 

\usepackage{hyperref}

\usepackage[accepted]{icml2023v2}

\usepackage{amsmath}
\usepackage{amssymb}
\usepackage{mathtools}
\usepackage{amsthm}

\usepackage[capitalize,noabbrev]{cleveref}

\theoremstyle{plain}
\newtheorem{theorem}{Theorem}[section]
\newtheorem{proposition}[theorem]{Proposition}

\theoremstyle{definition}
\newtheorem{definition}[theorem]{Definition}

\theoremstyle{remark}
\newtheorem{remark}[theorem]{Remark}

\usepackage[disable, textsize=tiny]{todonotes}

\usepackage{multirow} 
\usepackage[algo2e, noend, ruled]{algorithm2e} 
\usepackage{bm}

\DeclarePairedDelimiter{\norm}{\lVert}{\rVert}

\usepackage{enumitem}

\usepackage[hang,flushmargin]{footmisc}

\newcounter{sibcmntcounter}
\long\def\symbolfootnote[#1]#2{\begingroup
  \def\thefootnote{\fnsymbol{footnote}}\footnote[#1]{#2}\endgroup}
\newcommand{\sibcmnt}[1]{{\small\textbf{
      \textcolor{violet}{(C.\arabic{sibcmntcounter})}}
    \let\thefootnote\relax\footnotetext{\textcolor{violet}
        {\scriptsize(C.\arabic{sibcmntcounter})~ #1}}}
  \addtocounter{sibcmntcounter}{1}}

\icmltitlerunning{Differentially Private Distributed Bayesian Linear Regression with MCMC}

\begin{document}

\twocolumn[
\icmltitle{Differentially Private Distributed Bayesian Linear Regression with MCMC}




\begin{icmlauthorlist}
\icmlauthor{Barış Alparslan}{1}
\icmlauthor{Sinan Yıldırım}{1}
\icmlauthor{Ş. İlker Birbil}{2}
\end{icmlauthorlist}

\icmlaffiliation{1}{Faculty of Engineering and Sciences, Sabancı University, Turkey}
\icmlaffiliation{2}{Amsterdam Business School, University of Amsterdam, The Netherlands}

\icmlcorrespondingauthor{Barış Alparslan}{baris.alparslan@sabanciuniv.edu}

\icmlkeywords{Machine Learning, ICML}

\vskip 0.3in
]



\printAffiliationsAndNotice{}  

\begin{abstract}
We propose a novel Bayesian inference framework for distributed differentially private linear regression. We consider a distributed setting where multiple parties hold parts of the data and share certain summary statistics of their portions in privacy-preserving noise. We develop a novel generative statistical model for privately shared statistics, which exploits a useful distributional relation between the summary statistics of linear regression. We propose Bayesian estimation of the regression coefficients, mainly using Markov chain Monte Carlo algorithms, while we also provide a fast version that performs approximate Bayesian estimation in one iteration. The proposed methods have computational advantages over their competitors. We provide numerical results on both real and simulated data, which demonstrate that the proposed algorithms provide well-rounded estimation and prediction.
\end{abstract}

\section{Introduction} \label{introduction}
Linear regression is a mathematical method that lies at the core of statistical research. Many researchers have been working on linear regression since the 19th century, and hence, many well-known solution methods exist. On a separate note, privacy-preserving statistical learning has gained popularity and importance in recent years, with \emph{differential privacy} prevailing as the most commonly used definition for privacy \cite{Dwork_2006, dwork_2014_algorithmic, dankar_2013_DP}.  As a result, there is a recent but growing interest in differentially private linear regression.

Many works in the data privacy literature do not mainly focus on regression but are motivated by or can be applied to regression. As an example, differentially private empirical risk minimisation \citep{Chaudhuri_et_al_2009, Bassily_et_al_2014, abadi_et_al_2016, Nurdan_et_al_2022_DPopt} can be applied to regression once it is cast as a data-driven optimisation problem. Many general-purpose Bayesian differentially private estimation methods can also be used in regression problems. \citet{Williams_et_al_2010} is one of the first works that considered a hierarchical model for the privatised data and Bayesian estimation for the model parameters. \citet{Zhang_Rubinstein_Dimitrakakis_2016} analyse several differential privacy mechanisms for posterior sampling and suggest using these mechanisms also for linear regression. \citet{dimitrakakis2017differential} developed a posterior sampling query algorithm to combine differential privacy and Bayesian inference.  Contrary to those one-sample approaches, general-purpose differentially private Markov chain Monte Carlo (MCMC) algorithms, which aim to identify the whole posterior distribution via iterative sampling, can also be applied to regression \citep{Wang_2015_PrivacyMC, foulds_et_al_2016, Wang_2015_PrivacyMC, Yildirim_and_Ermis_2019, Heikkila_et_al_2019, Gong_2022, Alparslan_and_Yildirim_2022, Ju_et_al_2022}. 

Several works in the literature are somewhat more directly related to differentially private regression. \citet{Zhang_et_al_2012} have suggested a functional mechanism method, which is based on perturbing polynomial objective functions with privacy-preserving noise. As an alternative, \citet{Dwork_AnalyzeGauss_2014, Wang_2018_RevisitingDP} considered perturbation of summary statistics. \citet{Alabi_et_al_2022} provide a technical discussion on different point estimation methods for differentially private simple linear regression, (that is when we have a single feature). \citet{Ferrando_et_al_2022} present a method to compute confidence intervals for the coefficients of linear regression. \citet{Cai_et_al_2021} study the rates of convergence for parameter estimation with differential privacy via output perturbation, where a non-private estimator is perturbed. All those works consider point estimation of the linear regression parameters.

In this paper, we focus on differentially private distributed Bayesian inference for the parameters of linear regression. We use a novel hierarchical model that relies on a distributional relationship (Proposition \ref{prop: cond dist of Z given S}) between the summary statistics of linear regression, which, to the best of our knowledge, has not been exploited so far.  We propose Bayesian inference algorithms that take perturbations of summary statistics as observations. The general inferential tool we pick in this paper is MCMC, a well-known framework for iterative sampling from posterior distributions. As we shall see, the proposed MCMC algorithms in this paper already have lower computational complexities per iteration than their closest competitors in \citet{Bernstein_and_Sheldon_2019}. Additionally, we also propose much faster Bayesian estimation methods that perform estimation in one iteration. Finally, for the sake of generality, we assume a distributed setting where the total dataset is shared among multiple parties (data nodes), who want to collaborate for the inference of a common parameter, see \textit{e.g.}, \ \citet{Heikkila_et_al_2017} for such a setting. The non-distributed setting is just a special case (single data holder) for our methodology.

This paper has connections with several works in the literature, yet it has significant differences from each of those, as we shall explain below. 

For the privacy-preserving mechanism, we consider adding noise to summary statistics of linear regression, similarly to \citet{Wang_2018_RevisitingDP, Bernstein_and_Sheldon_2019}. The adaSSP framework of \citet{Wang_2018_RevisitingDP} also motivates the fast Bayesian estimation methods developed in this paper. However, adaSSP is a point estimation method while we aim for a posterior distribution. The latter work, \citet{Bernstein_and_Sheldon_2019}, is particularly related to this paper as they also study Bayesian linear regression with differential privacy using perturbed statistics of data. However, there are some important differences between our work and that of \citet{Bernstein_and_Sheldon_2019}. These differences stem from the choice of summary statistics and the consequent hierarchical structure used for modelling linear regression. Those modelling differences lead to significant differences in the inference methods as well as significant computational advantages for our methods. Specifically, the computational complexity of our methods is $\mathcal{O}(d^{3})$, where $d$ is the number of features. This order is much less than $\mathcal{O}(d^{6})$ of \citet{Bernstein_and_Sheldon_2019}. Finally, neither \citet{Wang_2018_RevisitingDP} nor \citet{Bernstein_and_Sheldon_2019} has considered a distributed learning setting as we do in this paper, although both works can be modified for the distributed setting after moderate modifications.

\citet{foulds_et_al_2016} and  \citet{Heikkila_et_al_2017} are other differentially Bayesian inference methods that target posterior distributions of perturbed summary statistics of sensitive data. \citet{Heikkila_et_al_2017} is particularly interesting because they consider a distributed setting and present linear regression as their showcase example. However, we differ from those works in the way we model the perturbed statistics and in the choice of inference methods. Specifically, \citet{foulds_et_al_2016, Heikkila_et_al_2017} treat the perturbed statistics as if not perturbed, while we correctly incorporate the effect of perturbation in our model.

Recently, \citet{Alparslan_and_Yildirim_2022}  and \citet{Ju_et_al_2022} employ data augmentation for modelling sensitive and privatised data and propose MCMC for Bayesian inference, the latter having linear regression as a major application. Their methods have $\mathcal{O}(n)$ complexity per iteration in general where $n$ is the number of instances in the data set, which can be slow when $n$ is large. In contrast, our methods are scalable in data size since their computational complexities do not depend on $n$. We note that \citet[Section 4.2]{Alparslan_and_Yildirim_2022} also present an MCMC method scalable with $n$ that exploits the approximate normality of additive summary statistics. However, a direct application of that would lead to an algorithm with $\mathcal{O}(d^{6})$ computational complexity (per iteration), like in \citet{Bernstein_and_Sheldon_2019}.

The paper is organised as follows: In Section \ref{sec: DP}, we review differential privacy. In Section \ref{sec: A model for differentially private distributed linear regression}, we lay out the hierarchical model for differentially private distributed linear regression with perturbed summary statistics. In Section \ref{sec: Algorithms for Bayesian inference of theta}, we present and discuss the aspects of the proposed inference algorithms. In Section \ref{sec: Numerical experiments}, we provide numerical experiments. We conclude in Section \ref{sec: Conclusion}. 

\textbf{Notation:} Matrices and vectors are shown in bold-face notation. For a matrix $\bm{A}$, its transpose, trace, and determinant (if they exist) are $\bm{A}^{T}$, $\text{tr}(\bm{A})$, and $\vert \bm{A} \vert$, respectively. $\bm{I}_{d}$ is the $d \times d$ identity matrix. For any sequence $\{a_{i}\}_{i \geq 0}$, we write $a_{i:j}$ for $(a_{i}, \ldots, a_{j})$. We write $x \sim P$ to mean the random variable $x$ has distribution $P$. $\mathcal{N}(\bm{m}, \bm{\Sigma})$ stands for the multivariate normal distribution with mean $\bm{m}$ and covariance $\bm{\Sigma}$. The Wishart and inverse-Wishart distributions, each with scale matrix $\bm{\Lambda}$ and $\kappa$ degrees of freedom, are shown as $\mathcal{W}(\bm{\Lambda}, \kappa)$ and $\mathcal{IW}(\bm{\Lambda}, \kappa)$, respectively. $\mathcal{IG}(a, b)$ stands for the inverse-gamma distribution with shape and scale parameters $a$ and $b$. We augment those notations, \textit{e.g.,} with $\bm{x}$, to denote the respective probability density functions (pdf), \textit{e.g.}, $\mathcal{N}(\bm{x}; \bm{m}, \bm{\Sigma})$. 

\section{Differential Privacy}
\label{sec: DP}

Differential privacy \citep{Dwork_2006, Dwork_2008} concerns randomised algorithms that run on sensitive, or usually private, data. A randomised algorithm takes an input data set $D \in \mathcal{D}$ and returns a random output in $\mathcal{O}$ where the randomness is intrinsic to the algorithm. A differentially private algorithm constrains the difference between the probability distributions of the output values obtained from neighbouring data sets. We say two data sets are \emph{neighbours} if they have the same size and differ by a single element, corresponding to a single individual's piece of data.
\begin{definition}[Differential privacy]
A randomised algorithm $M:\mathcal{D} \mapsto \mathcal{O}$ is $(\epsilon, \delta)$-differentially private (DP) if for any pair of neighbouring data sets $D, D' \in \mathcal{D}$ and for any subset $O \subseteq \mathcal{O}$ of the of support domain, it satisfies
\[
\mathbb{P}[M(D) \in O]\leq e^{\epsilon} \mathbb{P}[M(D') \in O] + \delta.
\]
\end{definition}
The definition implies that we have more privacy with smaller $(\epsilon, \delta)$ pair. Privacy-preserving algorithms often use noise-adding mechanisms. A popular noise-adding mechanism is the \textit{Gaussian mechanism} \citep{Dwork_et_al_2006}, which perturbs a function $f: \mathcal{D} \mapsto \mathbb{R}^{k}$ of the sensitive data, for some $k \geq 1$, with a random noise drawn from the Gaussian distribution. The amount of the added noise depends on the $L_{2}$-\textit{sensitivity} of the function, given by
\[
\Delta_{f} = \max_{\text{neighbour} D_{1},D_{2} \in \mathcal{D}} \norm{f(D_{1})-f(D_{2})}_{2}.
\]
An $(\epsilon, \delta)$-DP Gaussian mechanism returns
\begin{equation} \label{eq: Gaussian mechanism}
f(D) + \Delta_{f} \sigma(\epsilon, \delta)  \bm{v}, \quad \bm{v} \sim \mathcal{N}(\bm{0}, \bm{I}_{k})
\end{equation}
upon taking $D$ as the input, where the quantity $\sigma(\epsilon, \delta)$ ensures $(\epsilon, \delta)$-DP. In this work, we take $\sigma(\epsilon, \delta)$ as the analytical solution given in \citet[Algorithm 1]{Balle_and_Wang_2018} due to its tightness. The Gaussian mechanism is also central to other forms of privacy, such as zero-concentrated DP \citep{Bun_and_Steinke_2016} and Gaussian DP \citep{Dong_et_al_2022}. 

This paper considers $(\epsilon, \delta)$-DP as the type of privacy and the Gaussian mechanism to generate noisy observations. Moreover, the proposed methods in this paper never use the sensitive data given the noisy observations generated using the Gaussian mechanism, hence exploiting the \emph{post-processing} property of differential privacy \citep{Dwork_Algorithmic_2014}. 
\begin{theorem}[Post-processing] \label{thm: post-processing}
If $M: \mathcal{D} \mapsto \mathcal{O}$ be $(\epsilon,\delta)$-DP and let $f: O \rightarrow O'$ be another mapping independent of $D$ given $M(D)$. Then $f_{M}: \mathcal{D} \mapsto \mathcal{O}'$ with $f_{M}(D) = f(M(D))$ is $(\epsilon,\delta)$-DP. 
\end{theorem}

\section{Differentially Private Distributed Linear Regression} \label{sec: A model for differentially private distributed linear regression}

This section presents a new hierarchical model for differentially private distributed linear regression. For ease of exposition, we first present a model with a single data holder, then generalise the model for the distributed setting.

\subsection{Basic Model and Privacy Setup} \label{sec:noise_for_S_Z}
Suppose we have a sequence of random variables $\{ (\bm{x}_{i}, y_{i}): i = 1, \ldots, n \}$, where $\bm{x}_{i} \in \mathcal{X} \subseteq \mathbb{R}^{d \times 1}$ are the feature vectors and $y_{i} \in \mathcal{Y} \subseteq  \mathbb{R}$ is the $i$'th response variable. We consider the normal linear regression to model the dependency between $\bm{x}_{i}$ and $y_{i}$. Specifically, 
\[
y_{i} = \bm{x}_{i}^{T} \bm{\theta} + e_{i}, \quad e_{i} \overset{\text{i.i.d.}}{\sim} \mathcal{N}(0, \sigma_{y}^{2}), \quad i = 1, \ldots, n,
\]
where $\bm{\theta} \in \mathbb{R}^{d}$ is the vector of the linear regression coefficients. We assume that the feature vectors $\bm{x}_{i}$'s are i.i.d.\ with distribution $P_{\bm{x}}$. A particular case of interest will be one where $P_{\bm{x}}$ can be assumed to be a normal distribution. However, we will also present algorithms for general $P_{\bm{x}}$ that is not (even approximately) normal.

In matrix notation, the above can shortly be expressed as 
\[
\bm{y} = \bm{X} \bm{\theta} + \bm{e}, \quad \bm{e} \sim \mathcal{N}(\bm{0}, \sigma_{y}^{2} \bm{I}_{n}),
\]
where $\bm{X} = \begin{bmatrix} \bm{x}_{1}^{T} & \ldots & \bm{x}_{n}^{T} \end{bmatrix}^{T}$ is the so-called design matrix, $\bm{y} = \begin{bmatrix} y_{1} & \ldots & y_{n} \end{bmatrix}^{T}$. Additionally, we also define the summary statistics of $\bm{X}$ and $\bm{y}$ given by
\[
\bm{S} := \bm{X}^{T} \bm{X} \text { and } \bm{z} := \bm{X}^{T} \bm{y}, 
\]
respectively. In this paper, we assume a data-sharing scenario where $\bm{S}$ and $\bm{z}$ are privately released as the noisy summary statistics $\hat{\bm{S}}$ and $\hat{\bm{z}}$, constructed as
\begin{align}
\hat{\bm{S}} &= \bm{S} + \sigma_{s} \bm{M}, \label{eq: cond dist of hatS} \\
\hat{\bm{z}} &= \bm{z} + \sigma_{z} \bm{v}, \quad \bm{v} \sim \mathcal{N}(\bm{0}, \bm{I}_{d}), \label{eq: cond dist of hatZ}
\end{align}
where $\bm{M}$ is a $d\times d$ symmetric matrix with its upper triangular elements drawn from $\mathcal{N}(0, 1)$. \citet{Dwork_AnalyzeGauss_2014} arrange $\sigma_{s}$ and $\sigma_{z}$ so that both \eqref{eq: cond dist of hatS} and \eqref{eq: cond dist of hatZ} are $(\epsilon/2, \delta/2)$ differentially private, leading to $(\epsilon, \delta)$-DP overall. Differently than \citet{Dwork_AnalyzeGauss_2014}, we set
\[
\sigma_{s} = \sigma_{z} = \Delta_{sz} \sigma(\epsilon, \delta),
\]
where $\sigma(\epsilon, \delta)$ is given in \citet[Algorithm 1]{Balle_and_Wang_2018}, and $\Delta_{sz}$ is the overall $L_{2}$ sensitivity of $[\bm{S}, \bm{z}]$, given by
\[
\Delta_{sz} = \sqrt{\| X \|^{4} + \| X \|^{2} \| Y \|^{2}}
\]
with $\| X \| = \max_{\bm{x} \in \mathcal{X}} \| \bm{x} \|_{2}$ and $\| Y \| = \max_{y \in \mathcal{Y}} | y |$.

Finally, we assign prior distributions for $\bm{\theta}$, $\sigma_{y}^{2}$ as
\begin{equation} \label{eq: prior_theta_var_Y_sigma}
\bm{\theta} \sim \mathcal{N}(\bm{m}, \bm{C}), \quad \sigma_y^2 \sim \mathcal{IG}(a,b).
\end{equation}

Based on the above relations, we shall represent a hierarchical model that enables Bayesian inference of $\bm{\theta}$ given $\hat{\bm{S}}$ and $\hat{\bm{z}}$. One important element of our modelling approach is the following result that establishes the conditional distribution of $\bm{z}$ given $\bm{S}$, $\bm{\theta}$, and $\sigma_{y}^{2}$.
\begin{proposition} \label{prop: cond dist of Z given S}
For the normal linear regression model, we have
\[
\bm{z} | \bm{S}, \bm{\theta}, \sigma_{y}^{2} \sim \mathcal{N}(\bm{S} \bm{\theta}, \bm{S} \sigma_{y}^{2}).
\]
\end{proposition}
\begin{proof}
First, note that, 
\begin{align*}
\mathbb{E} [\bm{z} | \bm{X}, \bm{\theta}, \sigma_{y}^{2}] &= \mathbb{E}[\bm{X}^{T}\bm{X} \bm{\theta} + \bm{X}^{T} \bm{e} ] \\
 \text{Cov}(\bm{z}|\bm{X}, \bm{\theta}, \sigma_{y}^{2}) &= \bm{X}^T\bm{X}\sigma_{y}^2
\end{align*}
Hence, the conditional density of $\bm{z}$ given $\bm{X}$, $\bm{\theta}$, and $\sigma_{y}^{2}$ is
\begin{equation}\label{eq: z given X theta and sigmay}
p(\bm{z} | \bm{X}, \bm{\theta}, \sigma_{y}^{2}) = \mathcal{N}(\bm{z}; \bm{X}^{T}\bm{X} \bm{\theta}, \bm{X}^{T}\bm{X} \sigma_{y}^{2}).
\end{equation}
Let $\nu$ and $\omega$ denote the probability distributions of $\bm{x}$ and $\bm{S}$, respectively. By change of variables $\bm{S} = \bm{X}^{T}\bm{X}$, we have 
\begin{equation} \label{eq: change of variables}
\int  f(\bm{S}) \omega(\mathrm{d}\bm{S}) = \int f(\bm{X}^{T} \bm{X}) \nu(\mathrm{d}\bm{X})
\end{equation}
for any real-valued measurable function $f: \mathcal{S}_{d} \mapsto \mathbb{R}$, from the set $\mathcal{S}_{d}$ of $d \times d$ positive definite matrices. Also, for any pair of sets $A \subseteq \mathcal{S}_{d}$ and $B \in \mathbb{R}^{d \times 1}$, we can write
\begin{equation*} 
\mathbb{P}(\bm{S} \in A, \bm{z} \in B) = \int \mathbb{I}_{A}(\bm{X}^{T} \bm{X}) \mathbb{P}(\bm{z} \in B | \bm{X}) \nu(\mathrm{d}\bm{X}).
\end{equation*}
The integrand above depends on $\bm{X}$ \emph{through $\bm{X}^{T} \bm{X}$ only}, since $\mathbb{P}(\bm{z} \in B | \bm{X}) = \int_{B} \mathcal{N}(\bm{z}; \bm{X}^{T} \bm{X} \bm{\theta}, \bm{X}^{T} \bm{X} \sigma_{y}^{2}) \mathrm{d}\bm{z}$. Therefore, applying the relation in \eqref{eq: change of variables} to the RHS with the choice $f(\bm{S}) = \mathbb{I}_{A}(\bm{S}) \int_{B} \mathcal{N}(\bm{z}; \bm{S} \bm{\theta}, \bm{S} \sigma_{y}^{2}) \mathrm{d} \bm{z}$, we can rewrite the joint probability as
\begin{align*}
\mathbb{P}(\bm{S} \in A, \bm{z} \in B) &= \int \mathbb{I}_{A}(\bm{S}) \int_{B} \mathcal{N}(\bm{z}; \bm{S} \bm{\theta}, \bm{S} \sigma_{y}^{2}) \mathrm{d}\bm{z} \omega(\mathrm{d}\bm{S}) \\
&= \int_{A}  \int_{B} \mathcal{N}(\bm{z}; \bm{S} \bm{\theta}, \bm{S} \sigma_{y}^{2}) \mathrm{d}\bm{z} \omega(\mathrm{d}\bm{S}).
\end{align*}
This shows that the variables $\bm{z}, \bm{S}$ have the joint probability distribution $\omega(\mathrm{d}\bm{S}) \mathcal{N}(\bm{z}; \bm{S} \bm{\theta}, \bm{S} \sigma_{y}^{2}) \mathrm{d}\bm{z}$ and that the conditional probability density is $p(\bm{z} | \bm{S}, \bm{\theta}, \sigma_{y}^{2}) = \mathcal{N}(\bm{z}; \bm{S}\bm{\theta}, \bm{S} \sigma_{y}^{2})$, as claimed.
\end{proof}

At this point, some important modelling differences between our work and \citet{Bernstein_and_Sheldon_2019} are worth discussing. In \citet{Bernstein_and_Sheldon_2019}, the central limit theorem (CLT) is applied to $ \left[ \bm{S}, \bm{z}, \bm{y}^{T} \bm{y} \right]$, leading to a normality assumption for the whole vector. In contrast, we use the \emph{exact} conditional distribution $p(\bm{z} | \bm{S}, \bm{\theta}, \sigma^{2})$ thanks to Proposition \ref{prop: cond dist of Z given S}. Moreover, unlike \citet{Bernstein_and_Sheldon_2019}, we do \emph{not} require a noisy version $\bm{y}^{T} \bm{y}$, hence have a slight advantage of using less privacy-preserving noise. In summary, our model has a different hierarchical structure and requires less privacy-preserving noise.

\subsection{Distributed Setting} \label{sec: Distributed setting}
Here we extend our model to the distributed setting, where the total data are shared among $J \geq 1$ data holders as 
\begin{equation} \label{eq: partitioning the X y data}
(\bm{X}, \bm{y}) = \{ (\bm{X}_{j}, \bm{y}_{j});  j = 1, \ldots, J \}.
\end{equation}
We let $n_{i}$ be number of rows in each $\bm{x}_{i}$, so that $n = n_{1} + \ldots + n_{J}$. 
Each data holder $j$ shares their own summary statistics $\bm{S}_{j} = \bm{X}_{j}^{T} \bm{X}_{j}$, $\bm{z}_{j} = \bm{X}_{j}^{T} \bm{y}_{j}$ with privacy-preserving noise 
\begin{equation}
\begin{aligned} 
\hat{\bm{S}}_{j} &= \bm{S}_{j} + \sigma_{s} \bm{M}_{j}, \\
\hat{\bm{z}}_{j} &= \bm{z} + \sigma_{z}^{2} \bm{v}_{j}, \quad \bm{v}_{j} \sim \mathcal{N}(\bm{0}, \bm{I}_{d}).
\end{aligned}
\label{eq: noisy partitions}
\end{equation}
Note that, to preserve a given $(\epsilon, \delta)$-DP overall, each party must provide that level of privacy for their data, hence $\sigma_{s}^{2}$ and $\sigma_{z}^{2}$ are the same as before. The hierarchical structure of the overall model (specified for normally distributed $\bm{x}_{i}$'s) is shown in Figure \ref{fig: dplr_model}.
\begin{figure}[ht]
\centerline{
\includegraphics[scale = 0.85]{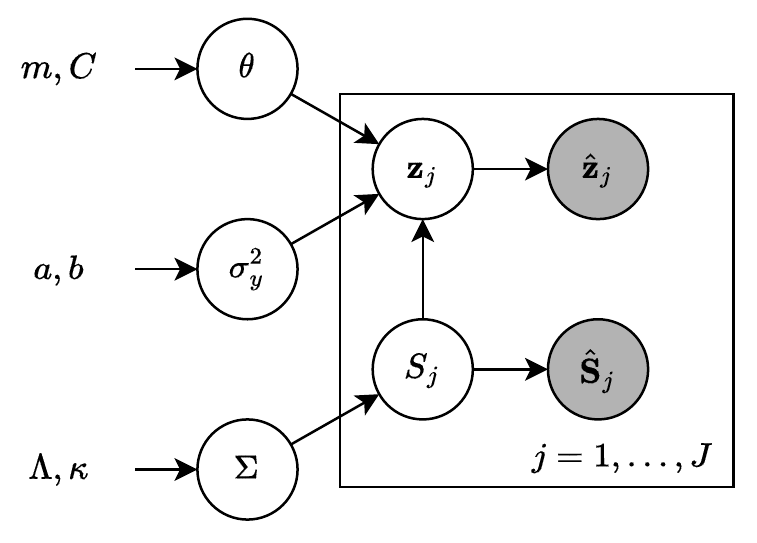}}
\caption{Differentially private distributed linear regression model (specified for normally distributed $\bm{x}_{i}$'s.)}
\label{fig: dplr_model}
\end{figure}

The distributed setting deserves separate consideration than the single data holder case for a couple of reasons: Firstly, the node-specific observations $(\hat{\bm{S}}_{1}, \hat{\bm{z}}_{1}), \ldots, (\hat{\bm{S}}_{J}, \hat{\bm{z}}_{J})$ are altogether statistically \emph{more informative} on $\theta$ than their aggregates $\sum_{j = 1}^{J} \hat{\bm{S}}_{j}$ and $\sum_{j = 1}^{J} \hat{\bm{z}}_{j}$. This is because the aggregate versions are \emph{not} sufficient statistics of the node-specific observations $(\hat{\bm{S}}_{1}, \hat{\bm{z}}_{1}), \ldots, (\hat{\bm{S}}_{J}, \hat{\bm{z}}_{J})$ with respect to $\theta$ (even when $\sigma_{y}^{2}$ is known.) Therefore, when the node-specific observations are available, one should not, in principle, trivially aggregate them and apply an inference method designed for $J = 1$ using those aggregates. 

Secondly, the partitioning of data as in \eqref{eq: partitioning the X y data} can be relevant to data privacy applications even \emph{outside} the distributed learning framework, rendering the methodology in Section \ref{sec: Algorithms for Bayesian inference of theta} useful in a broader sense. For example, batches of $(\bm{x}, y)$-type of data may be donated to a common data collector as in \eqref{eq: noisy partitions}. At this point, a particular and interesting relation exists with pan-privacy applications \citep{dwork2010pan}. Imagine that sensitive data from individuals are collected sequentially in time and the data holder is concerned about possible intrusions into the memory where the sensitive data are stored. Then, one possible way to ensure the privacy of the data against such possible intrusions, which is the promise of pan-privacy, is to store the noisy statistics of every new batch of data and erase the original sensitive data. Then, at any time, the data collector has data of the form $(\hat{\bm{S}}_{1}, \hat{\bm{z}}_{1}), \ldots, (\hat{\bm{S}}_{J}, \hat{\bm{z}}_{J})$, each pair corresponding to a batch, the $J$th pair being the current batch. As a result, inference algorithms in Section \ref{sec: Algorithms for Bayesian inference of theta} can be applied.

\section{Algorithms for Bayesian Inference} \label{sec: Algorithms for Bayesian inference of theta}
Bayesian inference targets the posterior distribution of the latent variables of the model, in particular $\bm{\theta}$, given the observations $\hat{\bm{S}}_{1:J}$ and $\hat{\bm{z}}_{1:J}$. We present several Bayesian inference algorithms for the hierarchical model described in the previous section. In addition to other concerns like computational budget, the choice among those approaches mainly depends on the specification of $P_{\bm{x}}$ as the distribution of $\bm{S}$ directly depends on it. In this paper, we consider the following two cases and devise algorithms for each of them:
\begin{enumerate}[leftmargin = *]
\item In some cases it may be adequate to specify $P_{\bm{x}} = \mathcal{N}(\bm{0}, \bm{\Sigma}_{x})$. This leads to $\bm{S} | \bm{\Sigma}_{x} \sim \mathcal{W}(\bm{\Sigma}_{x},n)$. Further, to account for the uncertainty about the covariance $\bm{\Sigma}_{x}$, one can treat it as a random variable with $\bm{\Sigma}_{x} \sim \mathcal{IW}(\bm{\Lambda}, \kappa)$. Figure \ref{fig: dplr_model} shows the hierarchical structure of the distributed setting with those specifications. We defer discussing the conflict between the normality and boundedness assumptions to Remark \ref{rem: normality and boundedness} towards the end of Section \ref{sec: MCMC with normally distributed features}.
\item As the second case, we assume a general (non-normal) $P_{\bm{x}}$. A normal approximation, based on the CLT, could be considered for the distribution $\bm{S}$ \citep{Andrew_Wishart_2011}. However, this would require the knowledge (or accurate estimation) of up to the fourth moments of $P_{\bm{x}}$ as well as expensive computations for sampling $\bm{S}$. We circumvent those difficulties by plugging in a point estimate of $\bm{S}$ given $\hat{\bm{S}}$ and use it during the sampling process as if it is the true $\bm{S}$ itself. Then, we develop two different algorithms for inference of $\bm{\theta}$, one being an MCMC algorithm and the other providing a closed form-solution for the posterior of $\bm{\theta}$ following a rough point-wise estimation of $\sigma_{y}^{2}$. Note that these algorithms with fixed $\bm{S}$ do not require a distribution for $\bm{x}$.

\end{enumerate}
Next, we provide the details of our approaches and the resulting algorithms.

\subsection{MCMC for Normally Distributed Features} \label{sec: MCMC with normally distributed features}

In this section, we present an MCMC algorithm for Bayesian inference for the differentially private distributed linear regression model when $P_{\bm{x}} = \mathcal{N}(\bm{0}, \bm{\Sigma}_{x})$ and $\bm{\Sigma}_{x} \sim \mathcal{IW}(\Lambda, \kappa)$. The latent variables involved in this variant are $\bm{\theta}, \bm{\Sigma}_{x}, \sigma_{y}^{2}, \bm{S}_{1:J}, \bm{z}_{1:J}$. Their posterior distribution given $\hat{\bm{S}}_{1:J}, \hat{\bm{z}}_{1:J}$ can be written as
\begin{align}
p(\bm{\theta},&\sigma_y^2,\bm{\Sigma}_{x},\bm{z}_{1:J}, \bm{S}_{1:J} | \hat{\bm{z}}_{1:J},\hat{\bm{S}}_{1:J}) \propto p(\bm{\theta})p(\sigma_y^2)p(\bm{\Sigma}_{x}) \nonumber \\
& \prod_{j = 1}^{J} p(\bm{z}_{j} | \bm{\theta}, \sigma_y^2, \bm{S}) p(\bm{S}_{j} |\bm{\Sigma}_{x}) p(\hat{\bm{S}}_{j}|\bm{S}_{j}) p(\hat{\bm{z}}_{j}|\bm{z}_{j}) .  \label{eq: joint distribution normal X} 
\end{align}
One could design an MCMC algorithm for this posterior distribution that updates $\bm{\theta}$, $\sigma_y^2$, $\bm{\Sigma}_{x}$, $\bm{z}_{1:J}$, $\bm{S}_{1:J}$ in turn based on their full conditional distributions. However, such an algorithm suffers from poor convergence because of a high posterior correlation between $\bm{\theta}$ and $\bm{z}_{1:J}$ (as verified in our numerical studies). It is well known that highly correlated variables result in poor convergence if they are updated one conditional on the other. To alleviate that problem, we work with the reduced model where $\bm{z}_{1:J}$ is integrated out. The reduced model has $\bm{\theta}, \bm{\Sigma}_{x}, \sigma_{y}^{2}$ as its latent variables, whose joint posterior distribution can be written as 
\begin{equation} \label{eq: reduced joint posterior}
\begin{aligned}
p(\bm{\theta},\sigma_y^2, &\bm{\Sigma}_{x}, \bm{S} | \hat{\bm{z}},\hat{\bm{S}}) \propto p(\bm{\theta}) p(\sigma_y^2) p(\bm{\Sigma}_{x}) 
\\ &\prod_{j = 1}^{J} p(\bm{S}_{j}|\bm{\Sigma}_{x})p(\hat{\bm{S}}_{j}| \bm{S}_{j}) p(\hat{\bm{z}}_{j}|\bm{S}_{j}, \bm{\theta}, \sigma_y^2),
\end{aligned}
\end{equation}
where  $p(\hat{\bm{z}}|\bm{S},\bm{\theta},\sigma_y^2) = \mathcal{N}(\hat{\bm{z}}; \bm{S} \bm{\theta}, \sigma_{y}^{2} \bm{S} \bm{\theta} + \sigma_{z}^{2} \bm{I}_{d})$. 

We would like to sample from the posterior distribution in \eqref{eq: reduced joint posterior} via MCMC that updates $\bm{\theta}$, $\sigma_y^2$, $\bm{\Sigma}_{x}$, $\bm{S}_{1:J}$ in turn based on their full conditional distributions. The variables $\bm{\theta}$ and $\bm{\Sigma}_{x}$ enjoy closed-form full conditional distributions (see Appendix \ref{sec: Derivations for MCMC-normalX} for the derivations):
\begin{align}
	\bm{\Sigma}_{x} | \bm{S}_{1:J}, \hat{\bm{S}}_{1:J}, \hat{\bm{z}}_{1:J} &\sim \mathcal{IW}\left(\bm{\Lambda} + \sum_{j=1}^{J} \bm{S}_{j}, \kappa+ n \right) \label{eq: Sigma full conditional}, \\
	\bm{\theta} | \sigma_{y}^{2}, \hat{\bm{z}}, \bm{S}_{1:J} & \sim \mathcal{N}(\bm{m}_p,\bm{\Sigma}_{p}), \label{eq: theta full conditional}
\end{align}
where the posterior moments for $\bm{\theta}$ are
\begin{align*}
\bm{\Sigma}_p^{-1} &= \sum_{j = 1}^{J} \bm{S}_{j} (\sigma_y^2 \bm{S}_j + \sigma_{z}^{2} \bm{I}_{d})^{-1}\bm{S}_j + \bm{C}^{-1}, \\
\bm{m}_p &= \bm{\Sigma}_{p} \left(\sum_{j = 1}^{J} \bm{S}_{j} (\sigma_{y}^{2} \bm{S}_{j} + \sigma_{z}^2 \bm{I}_{d})^{-1} \hat{\bm{z}}_{j} + \bm{C}^{-1}\bm{m}\right).
\end{align*}
The full-conditional distributions of $\bm{S}_{1:J}$ and $\sigma_{y}^{2}$ have no closed form; hence, we design Metropolis-Hastings (MH) moves to update them. For $\sigma_y^{2}$, one can simply use a random-walk MH move targeting $p(\sigma_{y}^{2}|\bm{\theta},\bm{S}_{1:J}, \hat{\bm{z}}_{1:J})$. For $\bm{S}_{1:J}$, their full conditional distribution can be factorised as
\begin{align*}\label{eq:S_update}
		& p(\bm{S}_{1:J}|\hat{\bm{S}}_{1:J}, \hat{\bm{z}}_{1:J}, \bm{\Sigma}_{x}, \sigma_{y}^{2}, \bm{\theta}) \\
		&\quad\quad\quad\quad\quad\quad\quad\quad = \prod_{j = 1}^{J} p(\bm{S}_{j} |\hat{\bm{S}}_{j}, \hat{\bm{z}}_{j},  \bm{\Sigma}_{x}, \sigma_y^2, \bm{\theta}),
\end{align*}
where each factor is given by
\begin{align*}
& p(\bm{S}_{j} |\hat{\bm{S}}_{j}, \hat{\bm{z}}_{j},  \bm{\Sigma}_{x}, \sigma_y^2, \bm{\theta}) \\
& \quad\quad\quad\quad\quad\quad \propto p(\hat{\bm{z}}_{j}|\bm{S}_{j},\bm{\theta},\sigma_y^2)p(\bm{S}_{j}|\bm{\Sigma}_{x})p(\hat{\bm{S}}_{j}|\bm{S}_{j}).\nonumber
\end{align*}
Thanks to that factorised form, each $\bm{S}_{j}$ can be updated with an MH move independently and in parallel. For the MH algorithm to update one $\bm{S}_{j}$, we propose a new value from a Wishart distribution as $\bm{S}_{j}' \sim \mathcal{W}(\bm{S}_j/\alpha, \alpha)$, which has mean $\bm{S}_{j}$ and variance determined by $\alpha > 0$. In our experiments, we adjust $\alpha$ using ideas from the adaptive MCMC framework \citep{Andrieu_and_Thoms_2008} to target an acceptance rate of around $0.2$. 

Algorithm \ref{alg: MCMC-normalX} represents the overall MCMC algorithm for the hierarchical model for differentially Bayesian distributed linear regression when $P_{\bm{x}}$ is a normal distribution with a random covariance matrix having an inverse-Wishart distribution. We call this algorithm \texttt{MCMC-normalX}.
\begin{algorithm}[t] 
\caption{\texttt{MCMC-normalX} - one iteration}
\label{alg: MCMC-normalX}	
\KwIn{Current values of $\bm{S}_{1:J}$, $\bm{\theta}$, $\sigma_y^2$, $\bm{\Sigma}_{x}$; observations $\hat{\bm{S}}_{1:J}$,$\hat{\bm{z}}_{1:J}$; noise variances $\sigma_{s}^{2}$, $\sigma_{z}^{2}$; proposal parameters $a$, $\sigma_{q}^{2}$; hyperparameters $a, b, \kappa, \bm{\Lambda}$, $\bm{m}$, $\bm{C}$.}
\KwOut{New sample of $\bm{\Sigma}_{x},\bm{S},\sigma_y^{2},\bm{\theta}$}	
 Sample $\bm{\Sigma}_{x}$ using \eqref{eq: Sigma full conditional}. 

\For{$j=1,2,\ldots J$}{
Update $\bm{S}_j$ via an MH move targeting $p(\bm{S}_{j} | \bm{\Sigma}_{x}, \bm{\theta}, \hat{\bm{z}}_{j})$.
}
Sample $\bm{\theta}$ using \eqref{eq: theta full conditional}.
  
Update $\sigma_y^2$ via an MH move targeting $p(\sigma_{y}^{2}|\bm{\theta},\bm{S}_{1:J}, \hat{\bm{z}}_{1:J})$.
\end{algorithm}

\begin{remark} \label{rem: normality and boundedness}
Admittedly, a potential concern is a conflict between the normality and boundedness assumptions (both for $\bm{x}$ and $y$). However, we also note that the collected data often happen to have some natural boundaries (which can be exploited to determine the sensitivity of the shared statistics), and yet the normal distribution is still used for modelling and subsequent inference mainly for the sake of tractability. With the normality assumption, one can implement computationally efficient algorithms at the expense of minor modelling inaccuracies. While we acknowledge the methodologies in \citet[Section 4.2]{Alparslan_and_Yildirim_2022} and \citet{Ju_et_al_2022} that can correctly incorporate the effect of truncation into inference, we remark that those methods pay the price of exactness by having $\mathcal{O}(n)$ computational complexity per iteration.
\end{remark}

\subsection{MCMC for Non-normally Distributed Features}
The normality assumption for $\bm{x}_{i}$'s in Section \ref{sec: MCMC with normally distributed features} may not be adequate for some data sets. Moreover, when $d$ is large, updating $\bm{S}_{j}$'s can be the bottleneck of \texttt{MCMC-normalX} in Algorithm \ref{alg: MCMC-normalX} in terms of computation time and convergence. We propose two algorithms to address both of those concerns. As it turns out, those algorithms provide accurate estimations even for the case of normally distributed features; see Section \ref{sec: Experiments with simulated data}. 

Our approach for non-normal $\bm{x}_{i}$'s is based on estimating $\bm{S}_{j}$'s from $\hat{\bm{S}}_{j}$s at the beginning, using some principled estimation method, and fixing $\bm{S}_{j}$'s to those estimates during the whole course of the inference procedure. In that way, we obtain a faster version of \texttt{MCMC-normalX} that is also well-suited for non-normal $\bm{x}_{i}$'s. Indeed, we observed in our experiments that this method outperforms the other methods for most of the cases, especially when the total number of nodes $J$ increases. We call this variant \texttt{MCMC-fixedS} and present it in Algorithm \ref{alg: MCMC-fixedS}.

As for estimating $\bm{S}_{j}$'s, one could simply take the privately shared $\hat{\bm{S}}_{j}$ as an estimator for $\bm{S}_{j}$, but  $\hat{\bm{S}}_{j}$ is not necessarily a positive (semi-)definite matrix. Instead, we consider the nearest positive semi-definite matrix to $\hat{\bm{S}}_{j}$  in terms of the Frobenius norm as the estimator of $\bm{S}_{j}$. (The nearest \emph{positive} definite matrix to $\hat{\bm{S}}_{j}$ does not exist.) To find the nearest positive semi-definite matrix, we follow \citet{Cheng_1998_psd_decomposition} and apply the following procedure for each $j = 1, \ldots, J$: (i) Calculate the eigendecomposition $\widehat{\bm{S}}_{j} = \bm{E}\bm{D}\bm{E}^T$, where $\bm{E}$ is a matrix of eigenvectors, and $\bm{D}$ is a diagonal matrix consisting of the eigenvalues $\lambda_{i}$. (ii) The nearest symmetric positive semi-definite matrix is $\widetilde{\bm{S}}_{j} = \bm{E}\bm{D}_{+}\bm{E}^T$, where $\bm{D}_{+}$ is a diagonal matrix with $\bm{D}_{+}(i, i) = \max\{ \bm{D}(i, i), 0 \}$. 

Note that $\widetilde{\bm{S}}_{j}$ found above is the maximum likelihood estimator of $\bm{S}_{j}$ given $\hat{\bm{S}}_{j}$ (over the set of positive semi-definite matrices) since the conditional distribution of $\hat{\bm{S}}_{j}$ given $\bm{S}_{j}$ is a normal distribution with mean $\bm{S}_{j}$.

Algorithm \ref{alg: MCMC-fixedS} is faster than Algorithm \ref{alg: MCMC-normalX}, since it avoids the step to update $\bm{S}_{j}$'s, which constitutes the main computational burden on Algorithm \ref{alg: MCMC-normalX}. However, Algorithm \ref{alg: MCMC-fixedS} can be made even faster by fixing $\sigma_{y}^{2}$ also. As a crude estimator, we used $\tilde{\sigma}_{y}^{2} = \| Y \| / 3$ throughout the experiments. We call the resulting algorithm \texttt{Bayes-fixedS-fast} and present it in Algorithm \ref{alg: Bayes-fixed S-fast}. Algorithm \ref{alg: Bayes-fixed S-fast} does nothing but calculate the moments of the posterior distribution of $\bm{\theta}$ given $\sigma_{y}^{2} = \tilde{\sigma}_{y}^{2} $, $\bm{S}_{j} = \widetilde{\bm{S}}_{j}$ and $\hat{\bm{z}}_{j}$ for $j = 1, \ldots, J$, and the prior parameters for $\bm{\theta}$.

\begin{algorithm}[t]
\caption{\texttt{MCMC-fixedS} - one iteration}
\label{alg: MCMC-fixedS}	
\KwIn{Current values of $\bm{\theta}$, $\sigma_y^2$; estimates $\widetilde{\bm{S}}_{1:J}$, observations $\hat{\bm{z}}_{1:J}$; noise variance $\sigma_{z}^{2}$, and hyperparameters $a$, $b$, $\bm{m}$, $\bm{C}$.}
\KwOut{New sample of $\sigma_y^{2},\bm{\theta}$.}
Use $\bm{S}_{1:J} = \widetilde{\bm{S}}_{1:J}$ throughout.

Sample $\bm{\theta}$ using \eqref{eq: theta full conditional}.

 Update $\sigma_y^2$ via an MH move targeting $p(\sigma_{y}^{2}|\bm{\theta},\bm{S}_{1:J}, \hat{\bm{z}}_{1:J})$.

\end{algorithm}

\begin{algorithm}[t] 
\caption{\texttt{Bayes-fixedS-fast}}
\label{alg: Bayes-fixed S-fast}
\KwIn{Observations $\hat{\bm{S}}_{1:J}$, $\hat{\bm{z}}_{1:J}$; noise variance: $\sigma_{z}^{2}$; estimate $\tilde{\sigma}_{y}^{2}$ of $\sigma_{y}^{2}$; hyperparameters: $\bm{m}$, $\bm{C}$.}
\KwOut{Estimate $\hat{\bm{\theta}}$.}
\For{$j = 1,2,\ldots J$}{
Calculate the estimate $\widetilde{\bm{S}}_{j}$ for $\bm{S}_{j}$ using $\hat{\bm{S}}_{j}$.

Calculate $\bm{U}_j = \widetilde{\bm{S}}_{j} (\tilde{\sigma}_{y}^{2} \widetilde{\bm{S}}_{j} +\sigma_{z}^2 \bm{I}_{d})^{-1} \widetilde{\bm{S}}_{j}$.

Calculate $\bm{u}_j = \widetilde{\bm{S}}_{j} ( \tilde{\sigma}_{y}^{2} \widetilde{\bm{S}}_{j} +\sigma_{z}^2 \bm{I}_{d})^{-1}\hat{\bm{z}}_j$.
}
\Return Posterior moments of $\bm{\theta}$: $\bm{\Sigma}_{\text{post}}^{-1} = \sum_{j = 1}^{J} \bm{U}_{j}  + \bm{C}^{-1}$, $\bm{m}_{\text{post}} = \bm{\Sigma}_{\text{post}} \left( \bm{C}^{-1} \bm{m} + \sum_{j = 1}^{J} \bm{u}_{j} \right)$.
\end{algorithm}

\subsection{Computational Cost} \label{sec: Computational cost}
All our methods described in this section require $\mathcal{O}(d^{3})$ computation (either per iteration for the iterative ones in Algorithms \ref{alg: MCMC-normalX} and \ref{alg: MCMC-fixedS} or as a whole for the fast version in Algorithm \ref{alg: Bayes-fixed S-fast}) since they deal with $d \times d$ matrices. In contrast, since \citet{Bernstein_and_Sheldon_2019} apply CLT to the vector $[\bm{S}, \bm{z}, \bm{y}^{T} \bm{y}]$, their methods deal with covariance matrices of size $(d^{2} + d + 1)$  \emph{explicitly}, which leads to $\mathcal{O}(d^{6})$ computation per MCMC iteration. For even moderate $d$, this computational difference between $\mathcal{O}(d^{6})$ and $\mathcal{O}(d^{3})$ becomes dramatic and the former may be prohibitive in practice. 

We also note that the complexity of our methods does not depend on the data size $n$. This is in contrast to the $\mathcal{O}(n)$ complexity of general-purpose methods applicable to linear regression, such as \citet[Section 4.3]{Alparslan_and_Yildirim_2022} and \citet{Ju_et_al_2022}.

\begin{figure*}[!ht]
    \centering
    \includegraphics[scale=0.5]{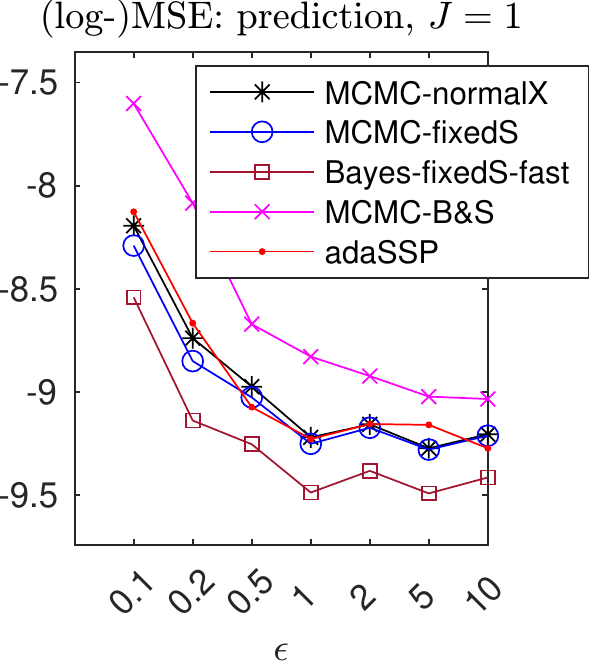} \includegraphics[scale = 0.5]{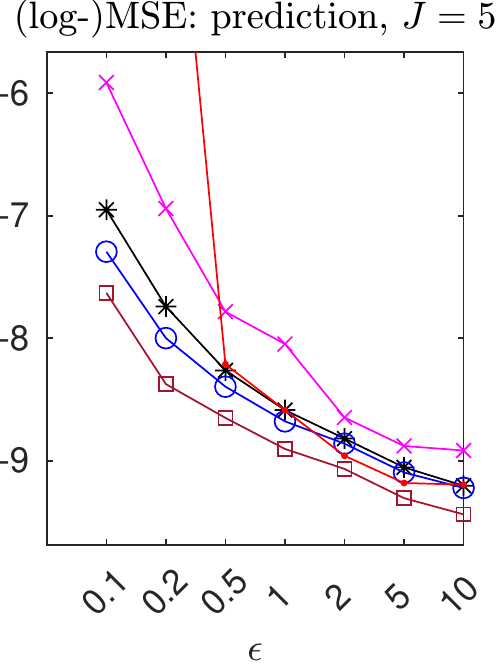} \includegraphics[scale=0.5]{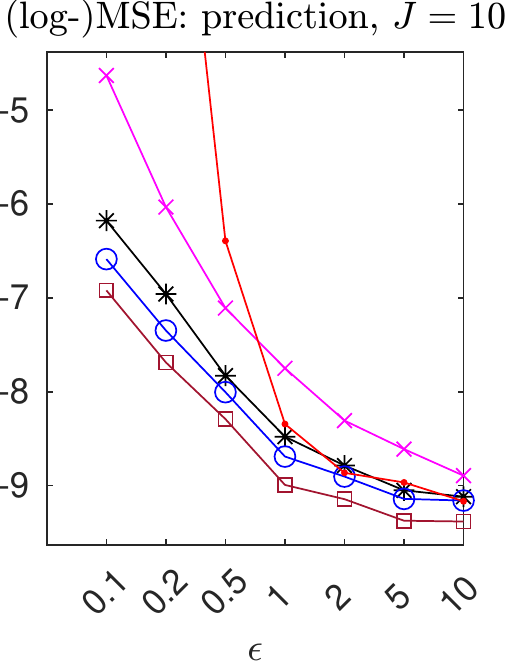} \includegraphics[scale=0.5]{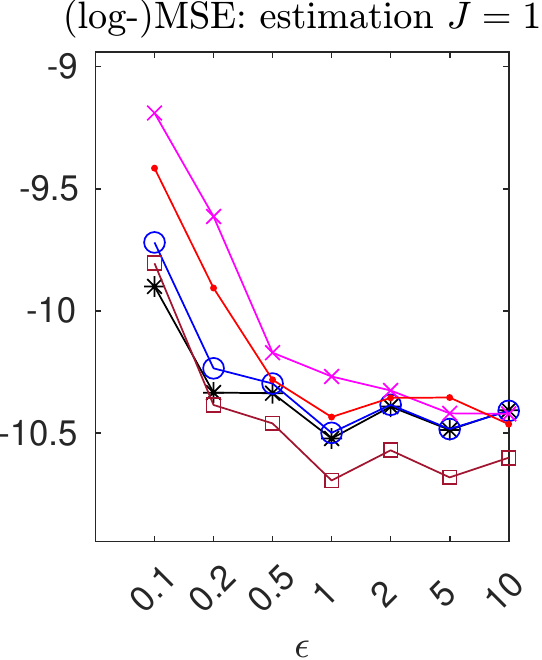} \includegraphics[scale=0.5]{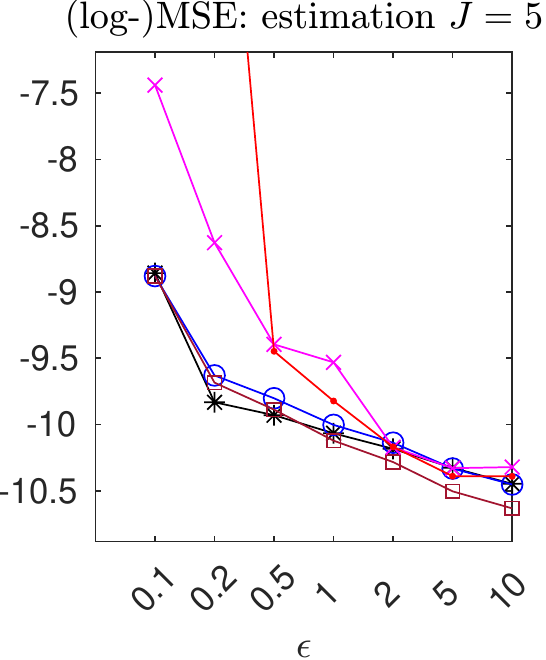} \includegraphics[scale=0.5]{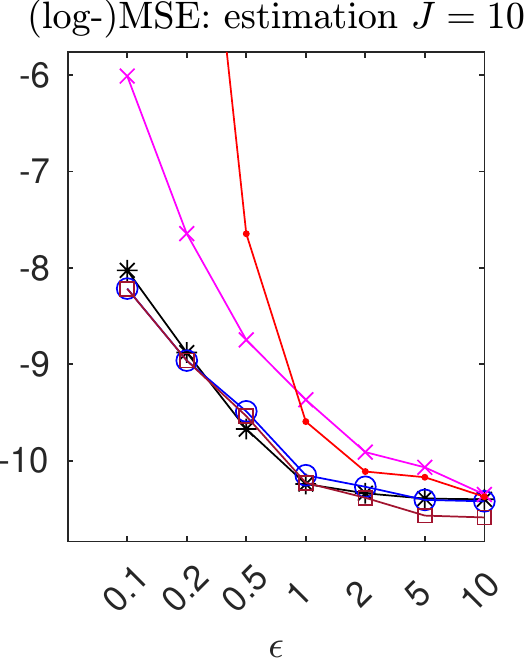} 
    \includegraphics[scale=0.5]{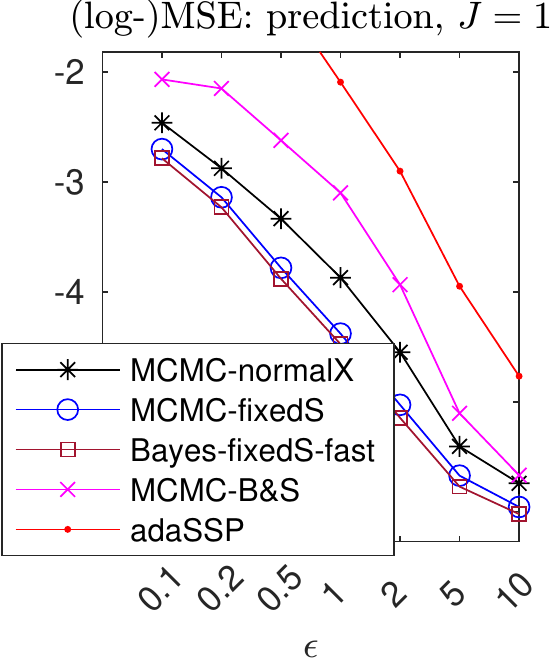} \includegraphics[scale = 0.5]{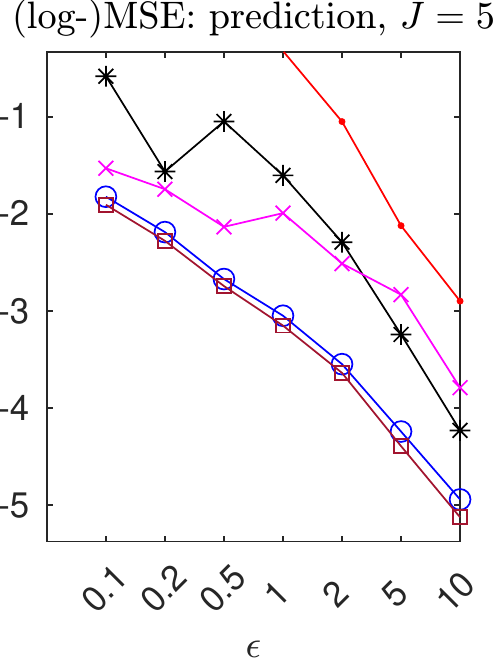} \includegraphics[scale=0.5]{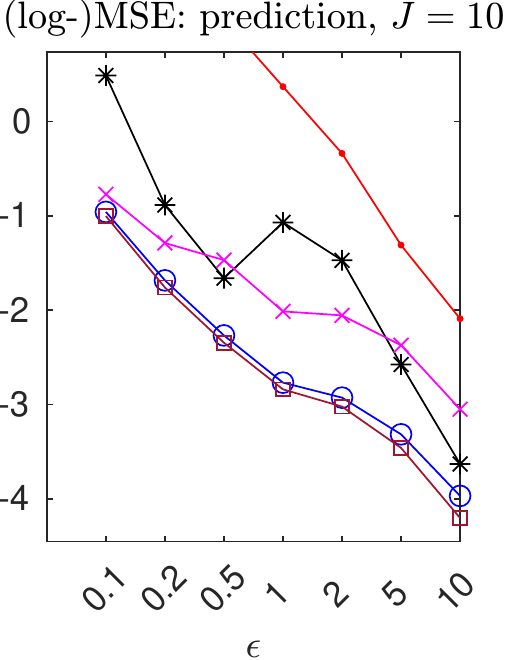} \includegraphics[scale=0.5]{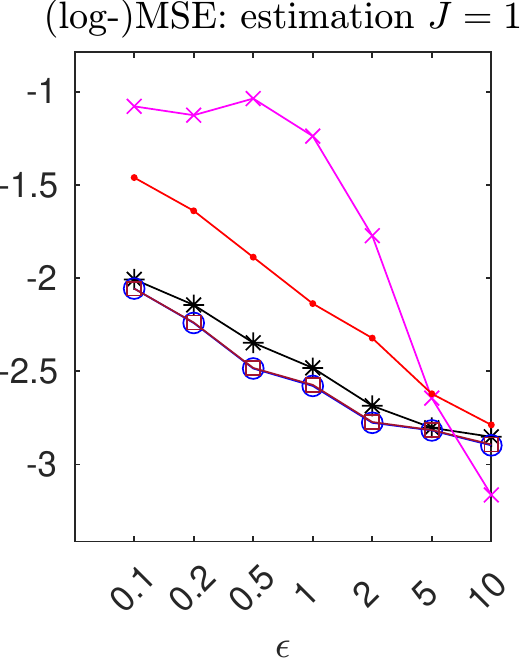} \includegraphics[scale=0.5]{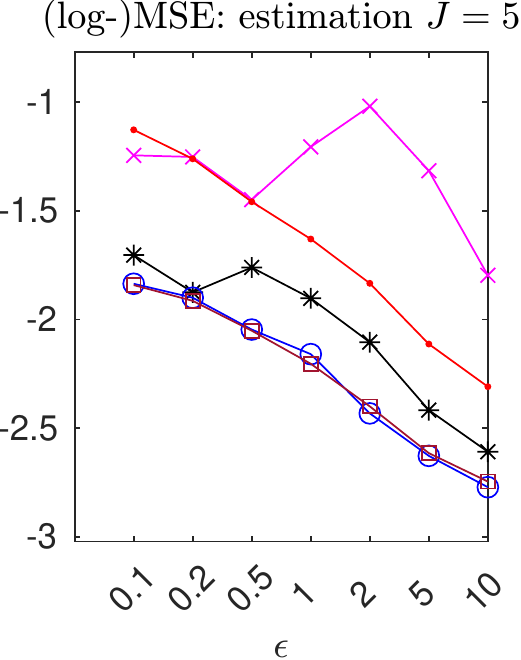} \includegraphics[scale=0.5]{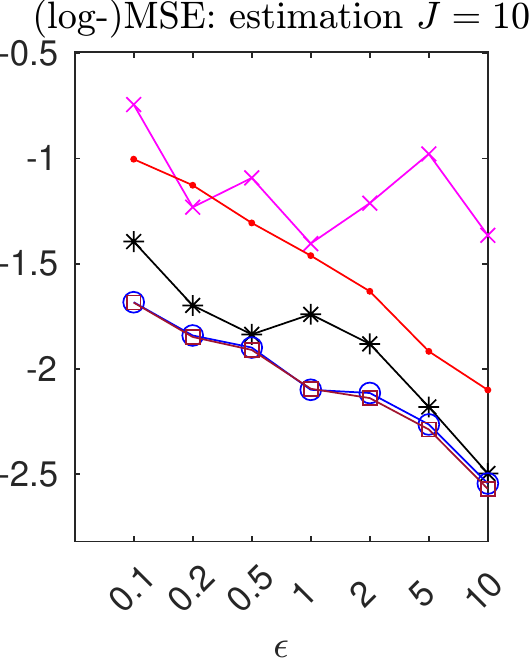}
    \caption{Averaged prediction and estimation performances (over 50 runs). Top row: $n = 10^{5}, d =2$, Bottom row: $n = 10^{5}, d =5$. For the non-private posterior, for $d = 2$ ($d = 5$, resp.), the obtained MSE values for estimation are  $2.92 \times 10^{-5}$ ($1.54\times 10^{-5}$, resp.) and for prediction are $9.60\times 10^{-5}$ ($4.19\times 10^{-5}$, resp.)}
    \label{fig: Est and Pred results}
\end{figure*}

\subsection{Extensions} \label{sec: Extensions}
We mention two other variants of our methodology, deferring the details to Appendix \ref{sec: Details of the other variants}. One solution for dealing with non-normal $P_{\bm{x}}$ could be to average the feature vectors in $\bm{X}$ and the corresponding response variables in $\bm{y}$, so that the averaged rows of $\bm{X}$ can be modelled as approximately normal, due to CLT. This enables using the methods devised for normally distributed features. For the details of this approach, see Appendix \ref{sec: Approximating normality by averaging}.

Secondly, if the features are normally distributed but not centred, we need to include the intercept parameter, which corresponds to appending $\bm{x}_{i}$ with a one from the left, and \texttt{MCMC-normalX} does not directly apply. In that case, we can modify the hierarchical model that accommodates the non-centralised features and the intercept parameter, and we still benefit from the sampling techniques involved in \texttt{MCMC-normalX} in Algorithm \ref{alg: MCMC-normalX}. Appendix \ref{sec: Including the Intercept} contains the details of the modified hierarchical model.

\begin{table*}[!ht]
\small
\caption{Averaged prediction performances (over 50 runs) for the real datasets - $\epsilon=1$} \label{table: table_all_results_prediction_eps_1}
\centerline{
\begin{tabular}{c|c|c|c|c|c|c}
    \toprule
   $J$ & \textbf{data sets} & \texttt{MCMC-normalX} & \texttt{MCMC-fixedS} & \texttt{Bayes-fixedS-fast}  & \texttt{MCMC-B\&S} & \texttt{adaSSP}\\
    \midrule
    \multirow{5}{*}{$J=1$} 
     & PowerPlant  & 0.0129 & 0.0129 & 0.0129 & \textbf{0.0128} & 0.0139  \\
     & BikeSharing  & 0.0024 & 0.0021 & 0.0021  & \textbf{0.0020} & 0.0107 \\
     & AirQuality  & 0.0060 &  \textbf{0.0057}  & \textbf{0.0057} & 0.0062 & 0.0066 \\
     & 3droad & 0.0229 & 0.0229 & 0.0229 & 0.0229 & 0.0229\\
    \midrule
    \multirow{5}{*}{$J=5$} 
     & PowerPlant  & \textbf{0.0133} & 0.0134 & 0.0134 &  0.0136 & 0.0235\\
     & BikeSharing  & 0.0174  & \textbf{0.0045} & \textbf{0.0045} & 0.0086 & 0.0382   \\
     & AirQuality  & 0.0142 & 0.0100 & \textbf{0.0099} & 0.0130 & 0.0227 \\
     & 3droad & 0.0229 & 0.0229 & 0.0229 &  0.0229 & 0.0229 \\
     \midrule
    \multirow{5}{*}{$J=10$} 
     & PowerPlant  & \textbf{0.0142} & 0.0143  & 0.0143  & 0.0143 & 0.0351 \\
     & BikeSharing  & 0.0812 & \textbf{0.0082} & \textbf{0.0082} & 0.0137 & 0.0526   \\
     & AirQuality  & 0.0985  & \textbf{0.0117} & \textbf{0.0117}  & 0.0216 & 0.0314\\
     & 3droad  & 0.0229  & 0.0229  & 0.0229 & 0.0229 & 0.0229 \\
     \bottomrule
\end{tabular}
}
\end{table*}

\section{Numerical Experiments} \label{sec: Numerical experiments}
We present several numerical evaluations of the proposed methods, \texttt{MCMC-normalX}, \texttt{MCMC-fixedS}, and \texttt{Bayes-fixedS-fast}, with simulated and real data. We compare our algorithms with two methods: \texttt{adaSSP} of \citet{Wang_2018_RevisitingDP} and the MCMC method of \citet{Bernstein_and_Sheldon_2019} for differentially private linear regression, which we will call \texttt{MCMC-B\&S}. Note that \texttt{adaSSP} and \texttt{MCMC-B\&S} were originally proposed for the non-distributed setting, that is, $J = 1$. For a comprehensive comparison, we implemented their extensions for $J \geq 1$. The details of those extensions are provided in Appendix \ref{sec: Details of the compared methods}. In particular, we carefully generalised the model in \citet{Bernstein_and_Sheldon_2019} for $J \geq 1$ (and for $(\epsilon, \delta)$-DP) in a similar fashion as we did to our model in Section \ref{sec: Distributed setting}. \texttt{MCMC-B\&S} is the adaptation of \citet[Algorithm 1]{Bernstein_and_Sheldon_2019} for this generalised model. Both for simulated and real data, we set $\|X\|$ and $\|Y\|$ to the maximum of the norms over the whole dataset (for real data, we perform centring and normalising first.) This procedure is equivalent to scaling the data to an interval and standard in existing work. The link for the code and data to replicate all of the experiments in this section is given at the end of Section \ref{sec: Conclusion}.

\subsection{Experiments with Simulated Data} \label{sec: Experiments with simulated data}
We considered two different configurations, $(n = 10^{5}, d = 2)$ and $(n = 10^{5}, d = 5)$, for the problem size. For each $(n, d)$ we simulated the data as follows. We generated $\bm{\theta} \sim \mathcal{N}(\bm{0}, \bm{I}_{d})$, $\bm{x}_{i} \sim \mathcal{N}(\bm{0}, \bm{\Sigma}_{x})$ where $\bm{\Sigma}_{x} \sim \mathcal{IW}(\bm{\Lambda}, \kappa)$, with $\kappa = d+1$ and the scale matrix selected randomly as $\bm{\Lambda} = \bm{V}^{T} \bm{V}$, where $\bm{V}$ is a $d \times d$ matrix of i.i.d.\ variables from $\mathcal{N}(0, 1)$. The response variables $\bm{y}$ are generated with $\sigma_{y}^{2} = 1$. For inference, we used the same $\bm{\Lambda}$, $\kappa$ as above and $a = 20$, $b = 0.5$, $\bm{m} = \bm{0}_{d \times 1}$, $\bm{C} = b/(a-1) \bm{I}_{d}$.

We evaluated the methods at all combinations of $J \in \{1, 5, 10\}$ and $\epsilon \in \{0.1, 0.2, 0.5, 1, 2, 5, 10\}$. All the MCMC algorithms were run for $10^{4}$ iterations. For each $(J, \epsilon)$ pair, we ran each method for $50$ times (each with different noisy observations) to obtain average performances.

For performance metrics, we looked at the mean squared errors (MSE) of (i) the estimates $\hat{\bm{\theta}}$ and (ii) the predictions $\hat{y}(\bm{x}_{\text{test}})$ generated by the methods. For the Bayesian methods, $\hat{\bm{\theta}}$ is taken as the posterior mean, which can be numerically estimated for the MCMC algorithms. For prediction performance, we calculated $\mathbb{E} [\hat{y}(\bm{x}_{\text{test}}) - y_{\text{test}} ]^{2}$. For the Bayesian methods, $\hat{y}(\bm{x}_{\text{test}})$ is the posterior predictive expectation of $y_{\text{test}}$ at $\bm{x}_{\text{test}}$. For \texttt{adaSSP}, we simply take $\hat{y}(\bm{x}_{\text{test}}) = \bm{x}_{\text{test}}^{T} \hat{\bm{\theta}}$. 

The results are summarised in Figure \ref{fig: Est and Pred results}. We observe that \texttt{MCMC-fixedS} and \texttt{Bayes-fixedS-fast} outperform \texttt{adaSSP} and \texttt{MCMC-B\&S} in almost all cases both in terms of estimation and prediction. Comparing the full-scale algorithms \texttt{MCMC-normalX} and \texttt{MCMC-B\&S} (that involve updates of $\bm{S}$), we observe a clear advantage of \texttt{MCMC-normalX} at $d = 2$ but \texttt{MCMC-B\&S} becomes more competitive at $d = 5$. This can be attributed to the fact that \texttt{MCMC-B\&S} requires the extra statistic $\bm{y}^{T} \bm{y}$, unlike \texttt{MCMC-normalX}, which causes \texttt{MCMC-B\&S} to use more noisy statistics. This difference becomes more significant at small $d$, where the relative effect of the presence of $\bm{y}^{T} \bm{y}$ on the sensitivity is more significant. Finally, all methods improve as $\epsilon$ grows, as expected.

We also compare the computation times of the MCMC algorithms \texttt{MCMC-normalX}, \texttt{MCMC-fixedS}, and \texttt{MCMC-B\&S}\footnote{The algorithms were run in MATLAB 2021b on an Apple M1 chip with 8 cores and 16 GB LPDDR4 memory.}. Figure \ref{fig: run_times} shows the run times of the algorithms vs. $d$. The drastic difference in computational loads explained in Section \ref{sec: Computational cost} is also visible in the figure. While \texttt{MCMC-B\&S} may be improved in terms of accuracy as $d$ increases, the $\mathcal{O}(d^{6})$ dramatically slows it down.
\begin{figure}[!h]
    \centering
    \includegraphics[scale=0.5]{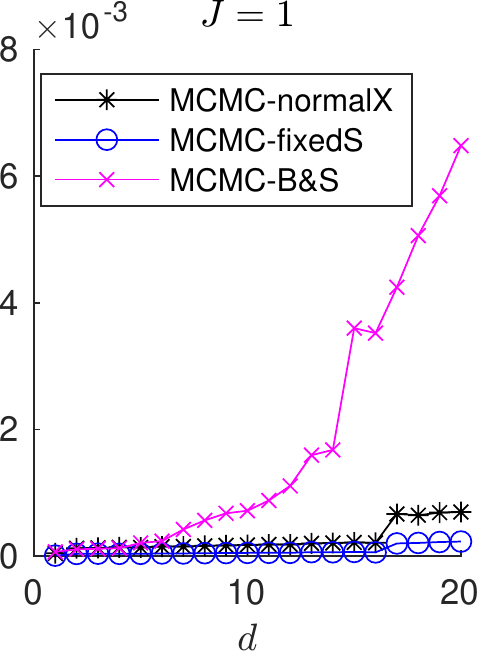}
    \includegraphics[scale=0.5]{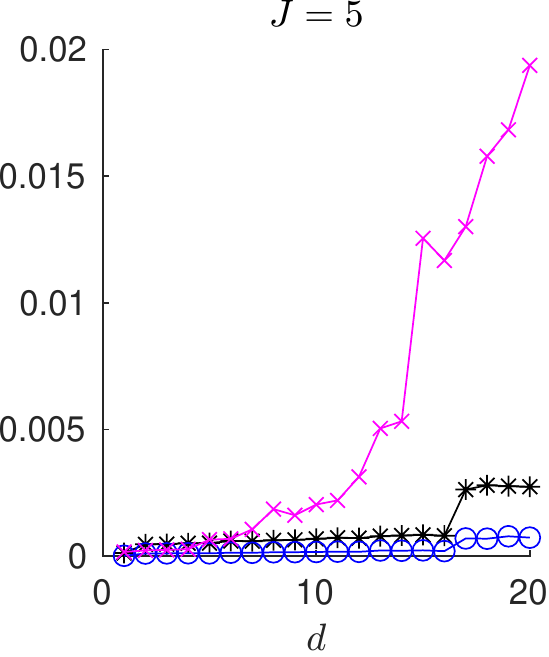}
    \includegraphics[scale=0.5]{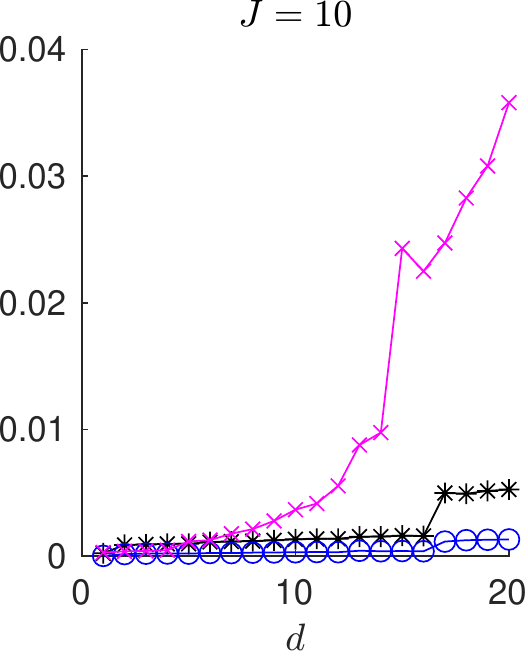}
    \caption{Run times per iteration for MCMC algorithms}
    \label{fig: run_times}
\end{figure}

In addition to the MSE, we also consider maximum mean discrepancy (MMD) for evaluating the calibration of the learned posteriors following the method in \cite{Gretton_et_al_2012}. The estimated MMD values in Figure \ref{fig: MMD plots} in Appendix \ref{sec: Additional results} show that all the methods converge to the non-private posterior as $\epsilon$ increases.

\subsection{Experiments with Real Data} \label{sec: Experiments with real data}
For the real data case, we use four different data sets from the UCI Machine Learning Repository. We disregard the columns including string data or key values (ID, name, date, etc.), and we consider the most right-hand column as $\bm{y}$. The finalised data sets are summarised below.
\centerline{
\small
\begin{tabular}{c|c|c|c}
\toprule
\textbf{data set} & $n$ & $d$ & \textbf{hyperlinks}\\
\midrule
power plant energy & 7655 & 4 &\href{https://archive.ics.uci.edu/ml/datasets/Combined+Cycle+Power+Plant}{view link}\\
bike sharing &  13904  & 14 & \href{https://archive.ics.uci.edu/ml/datasets/Bike+Sharing+Dataset}{view link}\\
air quality & 7486 & 12 &\href{https://archive.ics.uci.edu/ml/datasets/Air+Quality}{view link}\\
3d road   & 347900 & 3 & \href{https://archive.ics.uci.edu/ml/datasets/3D+Road+Network+\%28North+Jutland\%2C+Denmark\%29}{view link}\\
\bottomrule
\end{tabular}
}

For prediction, we took $80 \%$ of the data for training and the rest for testing. We present the average prediction performances (out of 50 runs) in Table \ref{table: table_all_results_prediction_eps_1} for each dataset and $J$ with $\epsilon = 1$.  We observe that the prediction performances of the compared methods are close, while \texttt{MCMC-fixed-S} and \texttt{Bayes-fixed-S} are arguably the most stable ones. When $J > 1$, (the distributed data setting) those two methods beat \texttt{adaSSP} and \texttt{MCMC-B\&S} more satisfactorily. To show the robustness of these conclusions to variation, we provide 90\% confidence intervals for the prediction performances in Table \ref{table: prediction CIs} in Appendix \ref{sec: Additional results}.

\section{Conclusion} \label{sec: Conclusion}
We propose a novel Bayesian inference framework, with MCMC being its main workhorse, for a differentially private distributed linear regression setting where the data is partitioned among the data holders and shared using summary statistics. We provide several Bayesian inference algorithms suited to the developed hierarchical model for linear regression. Those algorithms can be preferred one over the other depending on the computational budget, model specifics, or how much we know about the underlying statistical facts of the data. We exploit the conditional structure between the summary statistics of linear regression (Proposition \ref{prop: cond dist of Z given S}), which leads to feasible algorithms with computational advantages over their competitors. The numerical experiments show that the proposed methods are competitive with their state-of-the-art alternatives in terms of accuracy. The extensions mentioned in Section \ref{sec: Extensions} indicate potential future directions. 

The experiments demonstrate that \texttt{MCMCM-fixedS} and \texttt{Bayes-fixedS-fast} are competitive even under normality. Hence, we can suggest users try those fast and effective versions on their first attempts. However, \texttt{MCMC-normalX} can potentially provide more insight since it infers $P_{\bm{x}}$ as well. There is also room for improvement of \texttt{MCMC-normalX}. We chose the most common MH moves to update $\sigma_{y}^{2}$ and $\bm{S}_{j}$'s, without paying much attention to their efficiencies. Especially for large $d$, more advanced techniques such as Hamiltonian Monte Carlo or pseudo-marginal MCMC may be employed to facilitate the mixing of the algorithm. 

\paragraph{Code for the experiments:}
Link to the code and the data for the experiments: \url{https://github.com/sinanyildirim/Bayesian_DP_dist_LR.git}.


\section*{Acknowledgements}
The study was funded by the Scientific and Technological Research Council of Turkey (TÜBİTAK) ARDEB Grant No 120E534.




\bibliographystyle{icml2023}
\bibliography{my_refs}

\newpage
\appendix
\onecolumn

\section{Derivations for \texttt{MCMC-normalX}} \label{sec: Derivations for MCMC-normalX}
We reserve this section for the derivations for our algorithm \texttt{MCMC-normalX}.
\subsection{Full Conditional Distribution of $\bm{\Sigma}_{x}$}
We can write the full conditional distribution for $\bm{\Sigma}_{x}$ as
\begin{align}
	p(\bm{\Sigma}_{x}|\bm{\bm{S}_{1:J}},\hat{\bm{S}}_{1:J},\bm{\hat{\bm{z}}_{1:J}}) &\propto p(\bm{\Sigma}_{x})\prod_{j=1}^{J}p(\bm{S}_{j}| \bm{\Sigma}_{x})\nonumber\\
	&= \frac{|\bm{\Lambda}|^{dk/2}}{2^{dk/2}\bm{\Gamma}_{d}(\frac{k}{2})}|\bm{\Sigma}_{x}|^{-(d+\bm{\kappa}+1)/2}e^{-\frac{1}{2} \text{tr}(\bm{\Lambda}\bm{\Sigma}_{x}^{-1})}\nonumber
	\prod_{j=1}^{J}\frac{|\bm{S}_{j}|^{(n_j-d-1)/2}e^{-\frac{1}{2} \text{tr}(\bm{\Sigma}_{x}^{-1}\bm{S}_{j})}}{2^{n_jd/2}|\bm{\Sigma}_{x}|^{n_j/2}\bm{\Gamma}_{d}(n_j/2)}\nonumber\\
	&\propto |\bm{\Sigma}_{x}|^{-\frac{n}{2}-\frac{(d+\bm{\kappa}+1)}{2}} e^{-\frac{1}{2}(\sum \text{tr}(\bm{\Sigma}_{x}^{-1} \bm{S}_{j})+\text{tr}(\bm{\Lambda} \bm{\Sigma}_{x}^{-1}))} \nonumber\\
	&\propto |\bm{\Sigma}_{x}|^{-\frac{(d+\bm{\kappa}+n+1)}{2}} e^{-\frac{1}{2}\text{tr}((\sum \bm{\bm{S}_{j}} + \bm{\Lambda})\bm{\Sigma}_{x}^{-1})}.\nonumber
\end{align}
Therefore, we have
\begin{align*}
\bm{\Sigma}_{x} | \bm{S}_{1:J},\bm{\hat{\bm{S}}_{1:J}},\bm{\hat{z}}_{1:J} \sim \mathcal{IW}\left(\bm{\Lambda}+\sum_{j=1}^J \bm{S}_{j}, \bm{\kappa} + n \right).
\end{align*}
\subsection{Full Conditional Distribution of $\bm{\theta}$}
The full conditional distribution for $\theta$ can be written as
\begin{align*}
p(\bm{\theta}|\bm{S}_{1:J},\bm{\sigma}_y^2,\bm{\hat{z}}_{1:J}) &\propto \mathcal{N}(\bm{\theta}; \bm{m}, \bm{C}) p(\bm{\hat{z}}_{1:J}|\bm{S}_{1:J}, \bm{\theta}, \bm{\Sigma}_y^2). 
\end{align*}
For the second factor, we have
\begin{align*}
p(\bm{\hat{z}_{1:J}}|\bm{\bm{S}}_{1:J},\bm{\theta},\bm{\sigma}_y^2) &\propto \prod_{i=1}^{J}p(\hat{\bm{z}}_j|\bm{S}_j,\bm{\theta},\sigma_y^2) = \prod_{i=1}^{J}p(\hat{\bm{z}}_{j}; \bm{S}_j\bm{\theta}, \sigma_y^2\bm{S}_j + \sigma_{z}^{2} \bm{I}_{d})\\
&\propto \prod_{i=1}^{J} \exp \left\{-\frac{1}{2}(\hat{\bm{z}}_j-\bm{S}_j\bm{\theta})^T(\sigma_y^2\bm{S}_j + \sigma_{z}^{2} \bm{I}_{d})^{-1}(\hat{\bm{z}}_j-\bm{S}_j\bm{\theta}) \right\} \\
& \propto \exp \left\{-\frac{1}{2} \left[\bm{\theta}^{T} \left(\sum_j \bm{S}_j(\sigma_y^2 \bm{S}_j + \sigma_{z}^{2} \bm{I}_{d})^{-1}\bm{S}_j \right)\bm{\theta}-2\bm{\theta}^{T} \left(\sum_j \bm{S}_j (\sigma_y^2\bm{S}_j + \sigma_z^2 \bm{I}_{d})^{-1} \right)\bm{\hat{z}}_j \right] \right\}.
\end{align*}
Reorganising the terms, we end up with
\begin{align*}
p(\bm{\theta}|\bm{S}_{1:J},\bm{\sigma}_y^2,\bm{\hat{z}}_{1:J}) \propto \exp \left\{-\frac{1}{2} \left[ \bm{\theta}^{T} \bm{\Sigma}_{p}^{-1} \bm{\theta} -2\bm{\theta}^T \bm{\Sigma}_{p}^{-1} \bm{m}_{p} \right] \right\},
\end{align*}
where
$\bm{\Sigma}_{p} = [\sum_j \bm{S}_j(\sigma_{y}^2 \bm{S}_j + \sigma_{z}^{2} \bm{I}_{d})^{-1} \bm{S}_j + \bm{C}^{-1}]^{-1}$ and $m_p = \bm{\Sigma}_p[\sum_j \bm{S}_j (\sigma_{y}^{2} \bm{S}_{j} + \sigma_{z}^{2} \bm{I}_{d})^{-1})\hat{\bm{z}}_{j}+ \bm{C}^{-1} \bm{m}]$.
Therefore, $\bm{\theta} | \bm{S}_{1:J},\bm{\sigma}_y^2,\bm{\hat{z}}_{1:J} \sim \mathcal{N}(\bm{\theta}; \bm{m}_{p},\bm{\Sigma}_p)$.

\subsection{Acceptance Ratio for the MH Update of $\bm{S}_{j}$}
We drop the index $j$ from $\bm{S}_{j}$ for simplicity. When $\bm{S}' \sim \mathcal{W}(\bm{S}/\alpha, \alpha)$, the proposal density is
\begin{align*}
q(\bm{S}'|\bm{S}) &= \frac{|\bm{S}'|^{(\alpha-d-1)/2}e^{-\text{tr}[\alpha \bm{S}^{-1}\bm{S}']/2}}{|\bm{S}/\alpha|^{\alpha/2}2^{\alpha d/2}\Gamma_{d}(\frac{\alpha}{2})} = \frac{|\bm{S}'|^{(\alpha-d-1)/2}e^{-\text{tr}[\alpha \bm{S}^{-1}\bm{S}']/2}}{|\bm{S}|^{\alpha/2}2^{\alpha d/2}\Gamma_{d}(\frac{\alpha}{2})}\alpha^{\alpha/2}.
\end{align*}
Therefore, the acceptance ratio corresponding to this proposal is
\[
\min \left\{ 1, \frac{q(\bm{S}|\bm{S}')}{q(\bm{S}'|\bm{S})} \frac{\mathcal{W}(\bm{S}'; n_{j} \bm{\Sigma}_{x}, \kappa) p(\hat{\bm{S}} | \hat{\bm{S}}') \mathcal{N}(\hat{\bm{z}}; \bm{S}' \bm{\theta}, \sigma_{y}^{2} \bm{S} \bm{\theta} + \sigma_{z}^{2} \bm{I}_{d}) }{\mathcal{W}(\bm{S}; n_{j} \bm{\Sigma}_{x}, \kappa) p(\hat{\bm{S}} | \hat{\bm{S}}) \mathcal{N}(\hat{\bm{z}}; \bm{S} \bm{\theta}, \sigma_{y}^{2} \bm{S} \bm{\theta} + \sigma_{z}^{2} \bm{I}_{d})} \right\},
\]
where the ratio of proposals becomes
\begin{align}
\frac{q(\bm{S}|\bm{S}')}{q(\bm{S}'|\bm{S})} &= \frac{|\bm{S}|^{(\alpha - d-1)/2}|\bm{S}|^{\alpha/2}e^{-\text{tr}[\alpha \bm{S}'^{-1}\bm{S}]/2}}{|\bm{S}'|^{(\alpha - d-1)/2}|\bm{S}'|^{\alpha/2}e^{-\text{tr}[\alpha \bm{S}^{-1}\bm{S}']/2}} =\left(\frac{|\bm{S}|}{|\bm{S}'|} \right)^{\alpha -(d+1)/2}e^{\alpha(\text{tr}[\bm{S}^{-1}\bm{S}']-\text{tr}[\bm{S}'^{-1}\bm{S}])/2}.\nonumber
\end{align}

\subsection{Acceptance Ratio for the MH Update of $\sigma_{y}^{2}$}
To update $\sigma_{y}^{2}$, we use a random walk proposal $\sigma_{y}^{2\prime} \sim \mathcal{N}(\sigma_{y}^{2}, \sigma_{q}^{2})$. The resulting acceptance ratio is
\[
\min \left\{ 1, \frac{\mathcal{IG}(\sigma_{y}^{2\prime}; a, b) \prod_{j = 1}^{J}  \mathcal{N}(\hat{\bm{z}}_{j}; \bm{S}_{j} \bm{\theta}, \sigma_{y}^{2 \prime} \bm{S}_{j} \bm{\theta} + \sigma_{z}^{2} \bm{I}_{d})}{\mathcal{IG}(\sigma_{y}^{2}; a, b) \prod_{j = 1}^{J} \mathcal{N}(\hat{\bm{z}}_{j}; \bm{S}_{j} \bm{\theta}, \sigma_{y}^{2} \bm{S}_{j} \bm{\theta} + \sigma_{z}^{2} \bm{I}_{d}) } \right\}.
\]

\section{Other Variants} \label{sec: Details of the other variants}

This appendix is reserved for the details of the other variants mentioned in Section \ref{sec: Extensions}.

\subsection{Approximating Normality by Averaging} \label{sec: Approximating normality by averaging}
Assume $J = 1$ for simplicity. When the feature vectors $\bm{x}_{i}$, $i=1, \dots,n$ are not normal, another approach that we consider is based on modifying the data to make the rows of the feature matrix averages of multiple original features, and therefore, approximately normal, by the CLT. Specifically, let $k > 1$ be an integer that divides $n$ so that $m = n/k$ is also an integer. Consider the $m \times n$ matrix
\[
\bm{A} = \frac{1}{\sqrt{k}} \begin{bmatrix} \bm{1}_{1 \times k} & \bm{0}_{1 \times k} & \ldots & \bm{0}_{1 \times k} \\
\bm{0}_{1 \times k} & \bm{1}_{1 \times k} & \ldots & \bm{0}_{1 \times k} \\
\vdots & \vdots & \ddots & \vdots \\
\bm{0}_{1 \times k} & \bm{0}_{1 \times k} & \ldots & \bm{1}_{1 \times k} 
\end{bmatrix}_{m \times n},
\]
where $k$ is some number that divides $n$ so that $m = n/k$ is an integer. Then, the matrix $\bm{X}_{\text{av}} = \bm{A} \bm{X}$ corresponds to constructing a shorter $m \times d$ matrix whose $i$'th column is the average of the rows $(i-1)k + 1, \ldots, i k$ of $\bm{X}$ (scaled by $1/\sqrt{k}$). When $k$ is large enough, we can make normality assumptions for the rows of $\bm{X}_{\text{av}}$. Further, consider
\[
\bm{y}_{\text{av}} := \bm{A} \bm{y} = \bm{X}_{\text{av}} \bm{\theta} + \bm{A} \bm{e},
\]
whose mean is $\bm{X}_{\text{av}} \bm{\theta}$ and covariance $\bm{A} \bm{A}^{T} \sigma_{y}^{2}$. But we have $\bm{A} \bm{A}^{T} = \bm{I}_{m}$, so the covariance is $\sigma_{y}^{2} \bm{I}_{m}$. Therefore, the same hierarchical model in Figure \ref{fig: dplr_model} can be used for $\bm{X}_{\text{av}}$, $\bm{y}_{\text{av}}$ with their respective summary statistics
\[
\bm{z}_{\text{av}} = (\bm{X}_{\text{av}})^{T} \bm{y}_{\text{av}}, \quad \bm{S}_{\text{av}} = (\bm{X}_{\text{av}})^{T} \bm{X}_{\text{av}},
\]
as well as the noisy versions of those summary statistics to provide a given level of privacy. Note that $\bm{S}_{\text{av}}$ and $\bm{z}_{\text{av}}$ have the same sensitivities as $\bm{S}$ and $\bm{z}$, hence the same noise variances are needed for privacy. However, $\bm{S}_{\text{av}}$ and $\bm{z}_{\text{av}}$ bear less information about $\bm{\theta}$ than $\bm{S}$ and $\bm{z}$ due to averaging.

\subsection{Including the Intercept} \label{sec: Including the Intercept}
Again, assume $J = 1$ for simplicity. If we include the intercept parameter, which corresponds to appending $\bm{x}_{i}$ with a $1$ from the left, the design matrix will be changed from $\bm{S}$ to $
\bm{S}_{0} = \begin{bmatrix} n &n \bar{\bm{x}}^{T} \\  n \bar{\bm{x}} & \bm{S}  \end{bmatrix}$, 
where $\bar{\bm{x}} = \frac{1}{n} \sum_{i = 1}^{n} \bm{x}_{i}$. Also, note that $\bm{S} = (n-1) \widehat{\bm{\Sigma}}_{x} + n \bar{\bm{x}} \bar{\bm{x}}^{T}$ where $\widehat{\bm{\Sigma}}_{x}$ is the sample covariance of $\bm{x}_{1}, \ldots, \bm{x}_{n}$. Under the normality assumption for $\bm{x}_{i}$'s, we have $\bar{\bm{x}} \sim \mathcal{N}(\bm{m}, \bm{\Sigma}_{x}/n)$ and $\widehat{\bm{\Sigma}}_{x} \sim \mathcal{W}(n-1, \bm{\Sigma}_{x})$ are independent and have known distributions. Therefore, we can write a model that includes the additional variables $(\bm{b} = \bar{\bm{x}}, \widehat{\bm{\Sigma}}_{x}, \bm{S}_{0})$ such that $\bm{b}$ and $\widehat{\bm{\Sigma}}_{x}$ are independent and have known distributions and
\[
\bm{S}_{0} = \begin{bmatrix} n &n \bm{b}^{T} \\  n \bm{b} & (n-1) \widehat{\bm{\Sigma}} + n \bm{b} \bm{b}^{T} \end{bmatrix}
\] 
replaces $\bm{S}$ in the standard model. More specifically, we have the following hierarchical model:
\begin{align}
\bm{\theta} \sim \mathcal{N}(\bm{m}, \bm{C}), \quad \bm{\Sigma}_{x} \sim \mathcal{IW}(\bm{\Lambda}, \kappa) \nonumber, \quad \widehat{\bm{\Sigma}}_{x} | \bm{\Sigma}_{x} \sim \mathcal{W}(n-1, \bm{\Sigma}_{x}), \quad \bm{b} | \bm{\Sigma}_{x} \sim \mathcal{N}(\bm{\mu}, \bm{\Sigma}_{x}/n),\nonumber\\
\bm{z} | \bm{\theta},\sigma_{y}^{2}, \widehat{\bm{\Sigma}}_{x}, \bm{b} \sim \mathcal{N}(\bm{S}_{0} \bm{\theta}, \bm{S}_{0} \sigma_{y}^{2}), \quad
\hat{\bm{S}}| \widehat{\bm{\Sigma}}_{x}, \bm{b} = \mathcal{N}(\bm{S}_{0}, \sigma^{2}_{s, \epsilon} \bm{I}_{d+1}), \quad
\bm{\hat{z}}| \bm{z} = \mathcal{N}(\bm{z}, \sigma^{2}_{z} \bm{I}_{d+1}).\nonumber
\end{align}

\section{Compared Methods}  \label{sec: Details of the compared methods}

Here, we provide the details of the methods which we compare with the proposed methods in this paper. Those methods are originally proposed for $J = 1$. However, for comparison, we implemented their natural extensions to the general (distributed) case $J \geq 1$. The implementations of those methods can be found in the code package provided for this paper.

\subsection{MCMC of \citet{Bernstein_and_Sheldon_2019} Adapted to the Distributed Setting}
In \citet{Bernstein_and_Sheldon_2019}, $J = 1$ is considered only and the vector $\bm{ss} = [\text{vec}(\bm{S}), \bm{z} = \bm{X}^{T} \bm{y}, u = \bm{y}^{T} \bm{y}]$ is perturbed with privacy-preserving noise to generate the observations of the model. For $J \geq 1$, we consider the following natural extension for generating perturbed observations $\widehat{\bm{ss}} = [\text{vec}(\hat{\bm{S}}_{j}), \hat{\bm{z}}_{j}, \hat{u}_{j}]$, where
\begin{align}
\hat{\bm{S}}_{j} = \bm{S}_{j} + \sigma_{dp} \bm{M}_{j},
 \quad \hat{\bm{z}}_{j} = \bm{z}_{j} + \bm{v}_{j}, \quad \bm{v}_{j} \sim \mathcal{N}(\bm{0}, \sigma_{dp}^{2} \bm{I}_{d}) , \quad \hat{u}_{j} = u_{j} + w_{j}, \quad w_{j} \sim \mathcal{N}(0, \sigma_{dp}^{2}),
\end{align}
where $\sigma_{dp} = \sigma(\epsilon, \delta) \Delta_{\bm{ss}}$ with  $\Delta_{\bm{ss}} = \sqrt{\| X \|^{4} + \| X \|^{2} \| Y \|^{2} + \| Y \|^{4}}$. 

For completeness, we provide the further specifics of the model: We take $(\bm{\theta}, \sigma_{y}^{2}) \sim \mathcal{NIG}(a_{0}, b_{0}, \bm{m}, \bm{C})$ and $P_{\bm{x}} = \mathcal{N}(\bm{0}, \bm{\Sigma}_{x})$, where $\bm{\Sigma}_{x} \sim \mathcal{IW}(\bm{\Lambda}, \kappa)$. 

During the comparisons, we set the $a_{0}, b_{0}, \bm{m}, \bm{C}, \bm{\Lambda}, \kappa$ to the same values for both this model and our proposed model that assumes normally distributed features, \textit{i.e.}, $P_{\bm{x}} = \mathcal{N}(\bm{0}, \bm{\Sigma}_{x})$. Then, we apply an extension of the method of \citet{Bernstein_and_Sheldon_2019} suited to those observations. One iteration of that algorithm includes the following steps in order:
\begin{itemize}
\item Calculate the $D \times 1$ mean vector and $D \times D$ covariance matrix
\[
\bm{\mu}_{\text{prior}}(\bm{\theta}) = \mathbb{E}_{\bm{\theta}}[\bm{ss}], \quad  \bm{\Sigma}_{\text{prior}}(\bm{\theta}) = \text{Cov}_{\bm{\theta}}[\bm{ss}]
\]
This step requires the fourth moments $\mathcal{N}(\bm{0}, \bm{\Sigma}_{x})$. 
\item For $j = 1, \ldots, J$, sample $\bm{ss}_{j} \sim \mathcal{N}(\bm{\mu}_{\text{post}}^{(j)}, \bm{\Sigma}_{\text{post}}^{(j)})$ with
\[
\bm{\Sigma}_{\text{post}}^{(j)} = ([n_{j} \bm{\Sigma}_{\text{prior}}(\theta)]^{-1} +  (1/\sigma_{dp}^{2}) \bm{I}_{D})^{-1}, \quad \text{and} \quad \bm{\mu}_{\text{post}}^{(j)} = \bm{\Sigma}_{\text{post}}^{(j)} (\bm{\Sigma}_{\text{prior}}(\bm{\theta})^{-1} \bm{\mu}_{\text{prior}}(\bm{\theta}) + \widehat{\bm{ss}}_{j} / \sigma_{dp}^{2}).
\]
\item Sample $\bm{\Sigma}_{x} \sim \mathcal{IW}\left( \sum_{j = 1}^{J} \bm{S}_{j} + \bm{\Lambda}, n + \kappa \right)$.
\item Sample $(\bm{\theta}, \sigma^{2}) \sim \mathcal{NIG}(a_{n}, b_{n}, \bm{\mu}_{n}, \bm{\Lambda}_{n})$ by sampling $\sigma^{2} \sim \mathcal{IG}(a_{n}, b_{n})$, followed by sampling $\bm{\theta} \sim \mathcal{N}(\bm{\mu}_{n}, \sigma_{y}^{2} \bm{\Lambda}_{n}^{-1})$ with
\[
\bm{\Lambda}_{n} = \sum_{j = 1}^{J} \bm{S}_{j} + \bm{\Lambda}_{0}, \quad \bm{\mu}_{n} = \bm{\Lambda}_{n}\left( \sum_{j = 1}^{J} \bm{z}_{j} + \bm{\Lambda}_{0} \mu_{0}\right), \quad a_{n} = a_{0} + n/2, \quad b_{n} = 0.5 u + \bm{\mu}_{0}^{T} \bm{\Lambda}_{0} \bm{\mu}_{0} - \bm{\mu}_{n} ^{T} \bm{\Lambda}_{n} \bm{\mu}_{n}.
\]
\end{itemize}

\subsection{A Variant of \texttt{adaSSP} for the Distributed Setting} \label{sec: extended_adaSSP}
The \texttt{adaSSP} algorithm of \citep{Wang_2018_RevisitingDP} is originally designed for a single data holder. In \texttt{adaSSP}, a differentially private estimate of $\bm{\theta}$ is released as
\begin{equation} \label{eq: Wang_theta_est}
    \hat{\bm{\theta}}=(\hat{\bm{S}} + \lambda \bm{I}_{d})^{-1}\hat{\bm{z}}.
\end{equation}

Here $\hat{\bm{S}}$ and $\hat{\bm{z}}$ are the privatised versions of $\bm{S}$ and $\bm{z}$ as in \eqref{eq: cond dist of hatS} and \eqref{eq: cond dist of hatZ}, except that $\epsilon$ and $\delta$ must be changed to $2 \epsilon/3$ and $2 \delta / 3$ in those equations to provide $\epsilon, \delta$-DP. This is because \texttt{adaSSP} uses another parameter, $\lambda$, which is also calculated from the sensitive data and a third of the privacy budget is spent for privatising that calculation. With $v \sim \mathcal{N}(0,1)$, $\lambda$ is specifically calculated as
\[
\lambda = \max\left\{0,\sigma\sqrt{d \ln (6/\delta) \ln(2d^2/\rho)}- \Tilde{\lambda}_{\text{min}}\right\}
\]
with $\sigma = \| X\|^2/(\epsilon/3)$, $\lambda_{\text{min}}=\min(\text{eig}(\bm{S}))$, and $\Tilde{\lambda}_{\text{min}} = \max\left\{\lambda_{\text{min}}+ \sqrt{\ln(6/\delta)}\sigma v-\ln (6/\delta)\sigma v, 0\right\}$.
We consider an extension of  \citep{Wang_2018_RevisitingDP} for $J \geq 1$.  To perform the extension, we reflect on its tendency to approximate a (regularised) least square solution and consider the following estimate
\begin{equation} \label{eq: adaSSP estimate}
\hat{\bm{\theta}}=\left( \sum_{j = 1}^{J} \hat{\bm{S}}_j + \bm{I}_{d} \sum_{j = 1}^{J} \lambda_j  \right)^{-1} \left(\sum_{j = 1}^{J}\hat{\bm{z}}_j \right).
\end{equation}
Here $\hat{\bm{S}}_{j}$, $\hat{\bm{z}}_{j}$ and $\lambda_{j}$ are calculated in data node $j$ separately from the other nodes. The estimation procedure in \eqref{eq: adaSSP estimate} does not properly account for the Bayesian paradigm but aggregates the shared $\hat{\bm{S}}_{j}$'s and $\hat{\bm{z}}_{j}$'s to approximate the (regulated) least squares solution. Note that each node has separate $\lambda_j$ because it depends on $\delta_j$, and so the number of rows for each node.

\section{Additional results} \label{sec: Additional results}
Figure \ref{fig: MMD plots} shows the MMD estimates for $d = 2$. The (squared) MMD between two distributions can be estimated unbiasedly using i.i.d. samples from those distributions. Non-private posterior and private posteriors of \texttt{Bayes-fixedS-fast} are in closed form and can be sampled easily. For the MCMC models, we use every 50th sample of the chain to avoid autocorrelation and thus obtain nearly independent samples.
\begin{figure}[!ht]
    \centerline{
    \includegraphics[scale=0.3]{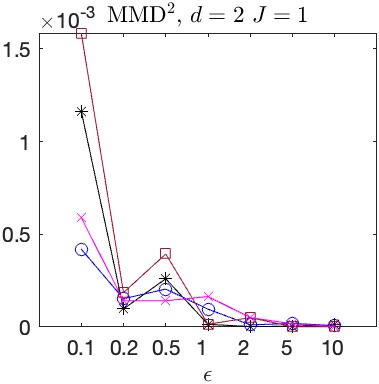} \includegraphics[scale=0.3]{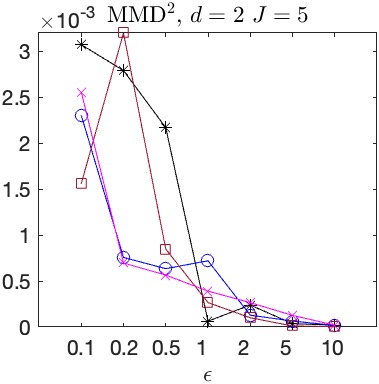}
    \includegraphics[scale=0.3]{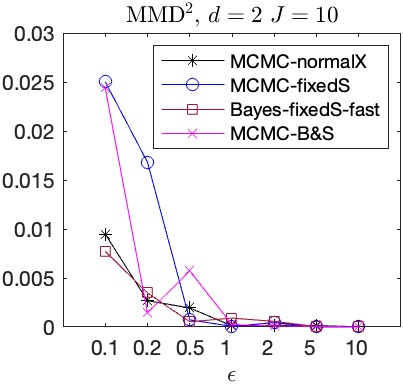}
    }
    \caption{MMD$^{2}$ estimates for each $J$ and $d=2$.}
    \label{fig: MMD plots}
\end{figure}

Table \ref{table: prediction CIs} shows the confidence intervals for the prediction MSE for the real-data experiments. 

\begin{table}[H]
\small
\caption{90\% confidence interval for mean prediction MSE (over 50 runs) for the real-data experiments - $\epsilon=1$} \label{table: prediction CIs}
\centerline{
\begin{tabular}{c|c|c|c|c|c|c}
    \toprule
   $J$ & \textbf{data sets} & \texttt{MCMC-normalX} & \texttt{MCMC-fixedS} & \texttt{Bayes-fixedS-fast}  & \texttt{MCMC-B\&S} & \texttt{adaSSP}\\
    \midrule
    \multirow{5}{*}{$J=1$} 
     & PowerPlant  & [0.0128, 0.0129] & [0.0128, 0.0129]	 & [0.0128, 0.0129] & [0.0128, 0.0129] & [0.0137, 0.0140]  \\
     & BikeSharing  & [0.0021, 0.0027] & [0.0018, 0.0024] & [0.0018, 0.0024]  & [0.0017, 0.0022] & [0.0106, 0.0108] \\
     & AirQuality  & [0.0051, 0.0069] &  [0.0048, 0.0066]  & [0.0048, 0.0066] & [0.0053, 0.0071]	& [0.0065, 0.0067]\\
     & 3droad & 	[0.0229, 0.0229] & [0.0229, 0.0229] & [0.0229, 0.0229] & [0.0229, 0.0229] & [0.0229, 0.0229]\\
    \midrule
    \multirow{5}{*}{$J=5$} 
     & PowerPlant  &[0.0132, 0.0135] & [0.0132, 0.0136] & [0.0132, 0.0136] & [0.0135, 0.0138] & [0.0234, 0.0236]\\
     & BikeSharing  &[0.0137, 0.0210] & [0.0041, 0.0049] & [0.0040, 0.0049]&[0.0076, 0.0095]& [0.0380, 0.0383]\\
     & AirQuality  &[0.0109, 0.0175]&[0.0089, 0.010]&[0.0089, 0.0109]&[0.0109, 0.0151]& [0.0226, 0.0229]\\
     & 3droad &[0.0229, 0.0229]& [0.0229, 0.0229]&[0.0229, 0.0229]	&[0.0229, 0.0229] & [0.0229, 0.0229]\\
     \midrule
    \multirow{5}{*}{$J=10$} 
     & PowerPlant  &[0.0139, 0.0145] &[0.0140, 0.0146]&[0.0140, 0.0146]&[0.0141, 0.0146]& [0.0349, 0.0353]\\
     & BikeSharing  &[0.0671, 0.0954]&[0.0072, 0.0092]&[0.0072, 0.0092]&[0.0116, 0.0158]& [0.0524, 0.0527]\\
     & AirQuality  & [0.0733, 0.1236]&[0.0099, 0.0135]& [0.0099, 0.0135]&[0.0175, 0.0257]& [0.0313, 0.0315]\\
     & 3droad  &[0.0229, 0.0229]&[0.0229, 0.0229]&[0.0229, 0.0229] &[0.0229, 0.0229]&[0.0229, 0.0229]
\\
     \bottomrule
\end{tabular}
}
\end{table}

\end{document}